\documentclass[twoside,11pt]{article}
\usepackage[hidelinks]{hyperref}
\usepackage{jair, theapa, rawfonts}

\usepackage{amsmath}
\usepackage{amssymb}
\usepackage{amsfonts}
\usepackage{amsthm}
\usepackage{url}
\usepackage{xspace}
\usepackage{color}
\usepackage[dvipsnames]{xcolor}
\usepackage{multirow}
\usepackage{algorithm}
\usepackage[noend]{algorithmic}
\usepackage{graphicx}
\usepackage[normalem]{ulem}
\usepackage{accents}
\usepackage{comment}
\usepackage{makecell}
\usepackage{enumitem}
\usepackage[T1]{fontenc}
\usepackage{forest}

\newif\iflong
\longtrue 

\newif\ifshort
\iflong\else\shorttrue\fi 

\newif\ifdraft
\drafttrue 

\newif\ifrecoverability

\newif\ifhideproofs

\ifhideproofs
    \usepackage{environ}
    \NewEnviron{hide}{}

\fi

\ifdraft
    \newcommand{\nb}[1]{\textcolor{red}{\bf!}%
    	\marginpar[\parbox{25mm}{\raggedleft\scriptsize\textcolor{red}{#1}}]%
    	{\parbox{25mm}{\raggedright\scriptsize\textcolor{red}{#1}}}}
    \else
    \newcommand{\nb}[1]{}
\fi

\newcommand{\A}{\mathcal{A}}
\newcommand{\C}{\mathcal{C}}
\newcommand{\D}{\mathcal{D}}

\newcommand{\G}{\mathcal{G}}

\renewcommand{\P}{\mathcal{P}}

\renewcommand{\S}{\mathcal{S}}
\newcommand{\tup}[1]{\langle #1 \rangle}

\newcommand{\atoms}{\textit{Atoms}}
\newcommand{\pred}{\textit{Pred}}
\newcommand{\aux}{\textit{aux}}
\newcommand{\auxp}[1]{\accentset{\ast}{#1}}
\newcommand{\subst}{\sigma}

\newcommand{\tablenote}[1]{\textsuperscript{(#1)}\xspace}
\newcommand{\lbcmLabel} {\tablenote{1}}
\newcommand{\ubwiLabel} {\tablenote{2}}
\newcommand{\lbakLabel} {\tablenote{3}}
\newcommand{\ubscLabel} {\tablenote{4}}
\newcommand{\ubfshLabel}{\tablenote{5}}
\newcommand{\ubakLabel} {\tablenote{6}}

\newcommand{\lbcm} {\hyperlink{lbcm} {\lbcmLabel}}
\newcommand{\ubwi} {\hyperlink{ubwi} {\ubwiLabel}}
\newcommand{\lbak} {\hyperlink{lbak} {\lbakLabel}}
\newcommand{\ubsc} {\hyperlink{ubsc} {\ubscLabel}}
\newcommand{\ubfsh}{\hyperlink{ubfsh}{\ubfshLabel}}
\newcommand{\ubak} {\hyperlink{ubak} {\ubakLabel}}

\newcommand{\lbcmTarget} {\protect\hypertarget{lbcm}{\lbcmLabel}}
\newcommand{\ubwiTarget} {\protect\hypertarget{ubwi}{\ubwiLabel}}
\newcommand{\lbakTarget} {\protect\hypertarget{lbak}{\lbakLabel}}
\newcommand{\ubscTarget} {\protect\hypertarget{ubsc}{\ubscLabel}}
\newcommand{\ubfshTarget}{\protect\hypertarget{ubfsh}{\ubfshLabel}}
\newcommand{\ubakTarget} {\protect\hypertarget{ubak}{\ubakLabel}}

\newcommand{\piduecomp}{$\pidue$}

\newcommand{\coNP}{coNP}

\newcommand{\PTIME}{PTIME}
\newcommand{\ffk}{full\xspace}

\newcommand{\dedineq}{DED$^{\ne}$\xspace}
\newcommand{\dedsineq}{DED$^{\ne}$s\xspace}
\newcommand{\inlinesubsection}[1]{\medskip \noindent \textbf{#1}}

\newcommand{\cqineq}{CQ\textsuperscript{$\ne$}\xspace}
\newcommand{\cqsineq}{CQ\textsuperscript{$\ne$}s\xspace}
\newcommand{\bcqineq}{B\cqineq}

\newcommand{\ucqineq}{U\cqineq}
\newcommand{\ucqsineq}{U\cqsineq}
\newcommand{\bucqineq}{B\ucqineq}
\newcommand{\bucqsineq}{B\ucqsineq}

\def\qedexample{\hfill{$\triangledown$} 
  \ifdim\lastskip<\medskipamount \removelastskip\penalty55\medskip\fi}

\newcommand{\facts}{\textit{Facts}}
\newcommand{\body}{\textit{body}}
\newcommand{\head}{\textit{head}}
\newcommand{\SD}{\tup{\Sigma,\D}}

\newcommand{\ineq}{\lambda}
\newcommand{\cq}{\textit{CQ}}
\newcommand{\vars}{\textit{Vars}}
\newcommand{\predatoms}{\textit{PA}}
\newcommand{\cnj}{\gamma}

\newcommand{\aczero}{\mathrm{AC}^0}
\newcommand{\pidue}{\mathrm{\Pi}^p_2}

\newcommand{\constr}{\Sigma}

\newcommand{\reps}{\mathsf{rep}_\constr}
\newcommand{\intrep}{\mathsf{intRep}_\constr}

\newcommand{\modelsopen}{\models}
\newcommand{\modelso}{\modelsopen}
\newcommand{\modelsar}{\models_\textsc{AR}}
\newcommand{\modelsiar}{\models_\textsc{IAR}}

\newcommand{\ars}{AR\xspace}
\newcommand{\irs}{IAR\xspace}

\newcommand{\ra}{\rightarrow}

\newcommand{\wrt}{\iflong with respect to\else w.r.t.\fi\xspace}

\newcommand{\seq}{\textit{Seq}}
\newcommand{\consistentprimed}{\Psi^\mathit{cons}}

\newcommand{\algirs}{\irs-CQEnt\xspace}

\newcommand{\algcomputerepairlinear}{Compute-Repair-Linear\xspace}

\newcommand{\queryentails}{\Psi^{iar}}%
\newcommand{\recoval}{\Psi^{rec}}%

\newcommand{\vseq}[1]{\mathbf{#1}}

\newcommand{\checkrepairacyclic}{\Psi^\mathit{rc}}

\newcommand{\chase}{\mathit{Chase}}
\newcommand{\cqforecovformula}{\Phi^{\mathit{rec}}}
\newcommand{\repaircheckingformula}{\Phi^{\mathit{rc}}}

\newcommand{\true}{\mathit{true}}
\newcommand{\raux}{R_{\mathit{aux}}}
\newcommand{\TGD}{\mathit{TGD}}

\newcommand{\clauses}{\mathit{CL}}
\newcommand{\bv}{\mathit{BV}}
\newcommand{\pv}{\mathit{PV}}
\newcommand{\mm}{\mathit{MM}}

\newcommand{\rclbStatement}[2]{
    There exists a set of #1 dependencies for which repair checking is #2-hard \wrt data complexity.}
\newcommand{\rcubStatement}[2]{
    Repair checking is in #2 \wrt data complexity in the case of #1 dependencies.}

\newcommand{\recovubStatement}[2]{
    Checking recoverability is in #2 \wrt data complexity in the case of #1 dependencies.}

\newcommand{\iclbStatement}[2]{
    There exist a set of #1 dependencies and a fact for which instance checking is #2-hard \wrt data complexity.}

\newcommand{\qelbStatement}[3]{
    There exist a set of #2 dependencies and a BUCQ for which #1-entailment is #3-hard \wrt data complexity.}
\newcommand{\qeubStatement}[3]{
    #1-entailment is in #3 \wrt data complexity in the case of #2 dependencies.}
\newtheorem{example}{Example}
\newtheorem{theorem}{Theorem}
\newtheorem{definition}{Definition}
\newtheorem{proposition}{Proposition}
\newtheorem{lemma}{Lemma}
\newtheorem{corollary}{Corollary}

\jairheading{??}{????}{????-????}{??/????}{??/????}

\newcommand{\mytitle}{Consistent Query Answering for Existential Rules with Closed Predicates}
\ShortHeadings{\mytitle}
{Marconi \& Rosati}

\firstpageno{0}

\begin{document}
	
\title{\mytitle}

\author{\name Lorenzo Marconi \email marconi@diag.uniroma1.it \\
	   \addr Sapienza Universit\`a di Roma, Rome, Italy
       \AND
       \name Riccardo Rosati \email rosati@diag.uniroma1.it  \\
        \addr Sapienza Universit\`a di Roma, Rome, Italy}

\maketitle

\begin{abstract}
  \emph{Consistent Query Answering} (CQA) is an inconsistency-tolerant approach to data access in knowledge bases and databases. 
  The goal of CQA is to provide meaningful (consistent) answers to queries even in the presence of \emph{inconsistent} information, e.g.\ a database whose data conflict with meta-data (typically the database integrity constraints).
  The semantics of CQA is based on the notion of \emph{repair}, that is, a consistent version of the initial, inconsistent database that is obtained through minimal modifications.
  We study CQA in databases with data dependencies expressed by existential rules. More specifically, we focus on the broad class of disjunctive embedded dependencies with inequalities (\dedsineq), which extend both tuple-generating dependencies and equality-generated dependencies.
  We first focus on the case when the database predicates are \emph{closed}, i.e.\ the database is assumed to have complete knowledge about such predicates, thus no tuple addition is possible to repair the database.
  In such a scenario, we provide a detailed analysis of the data complexity of CQA and associated tasks (repair checking) under different semantics (AR and IAR) and for different classes of existential rules. In particular, we consider the classes of acyclic, linear, full, sticky and guarded \dedsineq, and their combinations.
\end{abstract}

\sloppy

\section{Introduction}
\label{sec:introduction}

\emph{Consistent query answering}~\shortcite{B19} is an approach to inconsistency-tolerant reasoning in databases and knowledge bases (KBs).
The central notion of consistent query answering is the one of \emph{repair}. Given an \emph{inconsistent} database or KB $\D$, a repair of $\D$ is a consistent database or KB that is ``maximally close'' to $\D$. Several distinct formalizations of such notion of maximality and inconsistency are possible, giving rise to different types of repairs.
In general, many different repairs for the same database or KB may exist.
The problem of consistent query answering as originally formulated was: given a query $q$, compute the answers to $q$ that are true in every repair.
Also the \emph{repair checking} problem is often accounted, i.e.\ the problem of deciding whether a candidate $\D'$ is a repair of the original database or KB $\D$.
 
CQA was originally defined for relational databases~\shortcite{ABC99}, 
but it has then been studied also in the context of Description Logics KBs (see e.g.~\shortcite{R11,B12,BBG19}) and rule-based KBs, in particular knowledge bases with \emph{existential rules}~\shortcite{LMMMPS22,CalauttiGMT22}, where an extensional database (set of ground atoms) is coupled with a set of first-order implications.

In this paper, 
we consider a very expressive class of existential rules that comprises and extends both disjunctive tuple-generating dependencies (DTGDs) and equality-generating dependencies (EGDs), which are the most studied forms of rules studied in the literature.
Specifically, the  rules that we consider in this paper 
correspond to the so-called \emph{disjunctive embedded dependencies with inequalities} (\emph{\dedsineq})~\shortcite{DT05,APS16,D18}.  
We also consider five subclasses of these dependencies, corresponding to some of the most important ones studied in the literature on existential rules, i.e.\ the classes of \emph{acyclic}, \emph{full}, \emph{guarded}, \emph{linear}, and \emph{sticky} rules~\shortcite{CGL12,BLMS11,CGP12}.

Concerning the semantics, we first adopt a \emph{closed} interpretation of the predicates, and the well-known \emph{tuple-deletion} semantics for repairs~\shortcite{CM05,W05,AK09,CFK12}\iflong, according to which a \else: a \fi repair is a maximal (\wrt set inclusion) subset of the initial database that satisfies the rules.

\newcommand{\kim}{\mathsf{kim}}
\newcommand{\lou}{\mathsf{lou}}
\newcommand{\lee}{\mathsf{lee}}
\newcommand{\smith}{\mathsf{smith}}
\newcommand{\ssnone}{\mathsf{123}}
\newcommand{\ssntwo}{\mathsf{456}}
\newcommand{\slotone}{\mathsf{s}_1}
\newcommand{\bedone}{\mathsf{b}_1}
\newcommand{\icu}{\mathsf{icu}}
\begin{example}
    A hospital facility has complete knowledge about its domain of interest. Suppose that patients are modeled through a predicate $P$ whose arguments are the patient's SSN, first and last name. Moreover, a 
    necessary condition for being registered as a patient is having been visited (predicate $V$) by some physician in a certain date or having been hospitalized (predicate $H$) in some department for a certain time slot. Finally, every hospital bed reservation (predicate $R$) be allowed only for a single patient in a given time slot, and only for patients hospitalized for the same time slot and in the same department in which the bed is located (predicate $L$).

    The above rules can be formally expressed as follows:
    \[
    \begin{array}{r@{}l}
         \forall ssn,f,l\,(& P(ssn,f,l) \ra \exists ph,d\,V(ssn,ph,d) \lor \exists dep,ts\,H(ssn,dep,ts)) , \\
         \forall bed,ts,ssn_1,ssn_2\,(& R(bed,ssn_1,ts) \land R(bed,ssn_2,ts) \land ssn_1 \ne ssn_2 \ra \bot , \\
         \forall bed,ssn,ts\,(& R(bed,ssn,ts) \ra \exists dep,ts (H(ssn,dep,ts) \land L(bed,dep)) . \\
    \end{array}
    \]

    Consider now the following database:
    \[
    \begin{array}{r@{}l}
        \D = \{& P(\ssnone,\kim,\lee), P(\ssntwo,\lou,\smith), H(\ssntwo,\icu,\slotone), R(\ssnone,\icu,\slotone), R(\ssntwo,\icu,\slotone) \}
    \end{array}
    \]
    which sanctions that Kim Lee (SSN $\ssnone$) and Lou Smith (SSN $\ssntwo$) are patients, the latter has been hospitalized in the intensive care unit (ICU) for the time slot $\slotone$, and for both of them the bed $\bedone$ (located in the ICU department) has been reserved in the time slot $\slotone$.
    One can see that all the three rules are violated. For fixing the violation, we have to remove the fact $P(\ssnone,\kim,\lee)$ from the database, as Kim has not been visited by any physician nor hospitalized.
    Moreover, removing $R(\ssnone,\icu,\slotone)$ is necessary for satisfying the third dependency (and, as side effect, also the second one). Thus, in this case we have only one repair, i.e.\ 
    $
        \{ P(\ssntwo,\lou,\smith), H(\ssntwo,\icu,\slotone), R(\ssntwo,\icu,\slotone) \}
    $.
    \qedexample
\end{example}

Given these premises, we analyze the aforenamed decision problems
, that is, the repair checking and query entailment problems.
For the latter, we consider the language of Boolean unions of conjunctive queries with inequalities (\bucqsineq), and analyze two different 
inconsistency-tolerant semantics: \emph{AR-entailment}, i.e.\ skeptical entailment over all the repairs, and \emph{IAR-entailment}, i.e.\ entailment over the intersection of all the repairs. Notably, IAR-entailment corresponds to a sound approximation of AR-entailment.

We study the complexity of the above problems, providing a comprehensive analysis of the data complexity (i.e.\ the complexity with respect to the size of the database) of repair checking and both \ars-entailment and \irs-entailment of \bucqsineq, for all the possible combinations of the five classes of existential rules considered in this paper.
The results are summarized in Table~\ref{tab:results-all}.

Despite consistent query answering is well-known to be a computationally hard task, our results identify many interesting classes of existential rules in which the problems of consistent query answering and repair checking are tractable in data complexity. Moreover, for most of these tractable classes we prove the problems studied to enjoy the \emph{first-order (FO) rewritability}, i.e.\ they can be reduced to the evaluation of an FO sentence over the database.
This property is very interesting not only from the theoretical viewpoint, but also towards the definition of practical algorithms for consistent query answering based on FO rewriting methods, in a way similar to the case of classical query answering (see e.g.~\shortcite{KonigLMT15,KLM15}).


The closest work to the present paper is~\shortcite{LMMMPS22}, which presents an exhaustive study of the complexity of consistent query answering over existential rules. That paper inherits the open-world assumption (OWA) adopted by the recent approaches to existential rules, in which the interpretation of all predicates is open, thus repairing an inconsistency between the database and the rules is preferably done by (virtually) adding tuples to the initial database (that is, tuple deletions are considered only if tuple additions can not repair the data).
Conversely, our approach adopts a closed-world assumption (CWA), in which all predicates are closed and the only way to repair an inconsistency is through tuple deletions. Given this semantic discrepancy, the problems studied in the two papers are actually different, so
our findings can be seen as complementary to the ones presented in 
the aforementioned work (see also the conclusions).
We also remark that both the language for existential rules and the query language considered there are less expressive than the ones (disjunctive dependencies and \bucqsineq) studied in the present paper.

The paper is structured as follows. 
After a brief description of the main related work (Section~\ref{sec:related-work}), we introduce preliminary notions and definitions (Section~\ref{sec:preliminaries}) and we formalize the notion of repair and the decision problems that we focus on (Section~\ref{sec:repairs}).
In Section~\ref{sec:recoverability}, we provide some auxiliary results that will be useful for studying the problems of repair checking, \irs-entailment and \ars-entailment.
The computational complexity of such problems is then examined in Sections~\ref{sec:rc}, \ref{sec:intrep} and \ref{sec:allrep}, respectively.
Finally, in Section~\ref{sec:conclusions} we conclude the paper by discussing our results and proposing some possible future directions.

\iflong
\section{Related work}
\label{sec:related-work}

Consistent Query Answering was originally proposed for relational databases in~\shortcite{ABC99}, which introduced the notions of repairs and consistent answers, and repairs were intended as the the ones minimizing the symmetric difference with the original database.
Ever since, many works considered various kinds of integrity constraints and chose different repairing approaches.
Other types of repairs include the maximal consistent subsets of the original database~\shortcite{CM05}, as well as the minimal consistent supersets~\shortcite{CFK12}. The notion of maximality is also sometimes based on the cardinality~\shortcite{LB06}.
Here we define a repair as the maximal consistent superset of the original database, i.e.\ we adopt the so-called tuple-deletion semantics.

For what concerns the problem of repair checking, the most relevant papers that are strictly related to our investigation are~\shortcite{CM05,AK09,SC10,GO10,CFK12}, which deeply explored the problem for many classes of dependencies.

Also the problem of finding consistent answers for a given query under tuple-deletion semantics has been extensively studied by~\shortcite{CM05} and~\shortcite{CFK12}, with a particular focus on (some forms of) tuple-generating dependencies.
As shown in the conclusions, some of our results extend the ones provided in these two papers.
Special attention should also be paid to~\shortcite{FM05} and~\shortcite{KW17}, which presented FO-rewritable techniques for solving the CQA problem, though limiting the set of integrity constraints to (primary) key dependencies and making some assumptions on the user query.

In the mentioned works, all of which best reviewed in~\shortcite{C07,B19,W19}, the most common entailment semantics adopted corresponds to the one that here we call \ars-entailment.
The \irs semantics, instead, was previously studied in the context of ontologies by~\shortcite{LLRRS10}
, and further investigated for multiple DL languages in~\shortcite{LLRRS11,R11,LLRRS15}.
In particular, the latter proved \irs-entailment of CQs to be FO rewritable for the language $\text{DL-Lite}_{R,den}$.
Afterward,~\shortcite{B12} proposed a new semantics named ICR (``intersection of closed repairs''), which IAR is a sound approximation of, showing its FO expressibility for simple ontologies.

Both the entailment problems that we consider have also been studied in the field of existential rules~\shortcite{LMS12a,LMS12b,LMPS15,LMM18}, and systematically revised in~\shortcite{LMMMPS22}.
However, all such results are not straightforwardly transposable under CWA, which is our case study.
Indeed, in this paper we aim to provide a complementary analysis to the ``open'' case.
\fi
\iflong
\section{Preliminaries}
\label{sec:preliminaries}
\subsection{Databases, dependencies and queries}
\else
\section{Databases, dependencies and queries}
\label{sec:preliminaries}
\fi

\iflong
\subsubsection{Syntax}
\else
\inlinesubsection{Syntax}
\fi

We call \emph{predicate signature} (or simply \emph{signature}) a set of predicate symbols with an associated arity.
A \emph{term} is either a variable symbol or a constant symbol.
A \emph{predicate atom} (or simply \emph{atom}) is an expression of the form $p(\vseq{t})$, where $p$ is a predicate of arity $n$ and $\vseq{t}$ is a $n$-tuple of terms. We say that an atom $\alpha$ is \emph{ground} (or that $\alpha$ is a \emph{fact}) if no variable occur in it.
An \emph{inequality atom} (or just \emph{inequality}) is an expression of the form $t\neq t'$, where $t$ and $t'$ are terms.
Given a first-order (FO) formula (or a set thereof) $\phi$, we denote as $\vars(\phi)$ the set of variables occurring in $\phi$.
Moreover, given an FO theory $\Phi$, we denote by $\pred(\Phi)$ the set of predicates occurring in $\Phi$.

Given a predicate signature $\S$, a \emph{DB instance} (or simply \emph{database}) $\D$ for $\S$ is a set of facts over the predicates of $\S$.
Hereinafter, we assume w.l.o.g.\ that a signature always contains the special predicate $\bot$ of arity 0, and that no database contains the fact $\bot$.

\begin{definition}[CQ and BCQ]
    \label{def:cq}
    A \emph{conjunctive query with inequalities (\cqineq)} $q$ is an FO formula of the form $\exists \vseq{y}\, (\alpha_1\wedge\ldots\wedge\alpha_k\wedge\ineq_1\wedge\ldots\wedge\ineq_h)$, where $k\geq 0$, $h\geq 0$, $k+h\geq 1$, every $\alpha_i$ is a predicate atom, every $\ineq_i$ is an inequality, and $\vseq{y}$ is a sequence of variables each of which occurs in at least one conjunct.
    Given a \cqineq $\exists\vseq{y}\,(\cnj)$, we denote by $\predatoms(q)$ the set of predicate atoms occurring in $\cnj$.
    If every variable of $q$ occurs in at least one predicate atom, we say that $q$ is a \emph{safe} \cqineq.
    Moreover, if no variable of $q$ is free (i.e.\ all such variables occur in $\vseq{y}$), we say that $q$ is a \emph{Boolean \cqineq (\bcqineq)}.
\end{definition}

\begin{definition}[UCQ and BUCQ]
    \label{def:ucq}
    A \emph{union of conjunctive queries with inequalities (\ucqineq)} $q$ is an FO formula of the form
    $\bigvee_{i=1}^m q_i$,
    where $m\geq 1$ and every $q_i$ is a \cqineq.
    We denote by $\cq(q)$ the set of \cqsineq occurring in $q$, i.e.\ $\cq(q)=\{q_1,\ldots,q_m\}$.
    If all such $q_i$ are safe and share the same free variables, we say that $q$ is a \emph{safe} \ucqineq.
    Moreover, if every $q_i$ is Boolean, we say that $q$ is a \emph{Boolean \ucqineq (\bucqineq)}.
\end{definition}

We also refer to \cqsineq and \ucqsineq without inequalities simply as CQs and UCQs, respectively.
Given a \ucqineq $q$ with free variables $\vseq{x}$, and given a tuple of constants $\vseq{t}$ of the same arity as $\vseq{x}$, we denote by $q(\vseq{t})$ the \bucqineq obtained from $q$ by replacing all the occurrences of free variables with the constants in $\vseq{t}$.

We are now ready to define the notion of dependency that we use throughout the paper.\footnote{Although slightly different, Definition~\ref{def:dependency-new} is actually equivalent to the notion of disjunctive embedded dependency with inequalities \emph{and equalities} presented in~\shortcite{D18,DT05}.}

\begin{definition}[Dependency]
\label{def:dependency-new}
    Given a predicate signature $\S$, a \emph{disjunctive embedded dependency with inequalities} (\dedineq), or simply \emph{dependency}, for $\S$ is an FO sentence over the predicates of $\S$ of the form:
    \begin{equation}
    \label{eqn:dependency}
        \forall \vseq{x}\, \big(\cnj \rightarrow q \big) ,
    \end{equation}
    where $\vseq{x}$ is a sequence of variables such that $\exists \vseq{x}\,(\cnj)$ is a safe \bcqineq,
    and $q$ is a \ucqineq whose free variables occur in $\vseq{x}$,
    and such that $q(\vseq{t})$ is a safe \bucqineq for any tuple $\vseq{t}$ of constants of the same arity as $\vseq{x}$.
\end{definition}

Given a predicate signature $\S$ and a dependency $\tau$ for $\S$ of the above form (\ref{eqn:dependency}), we indicate with $\body(\tau)$ 
the \cqineq $\cnj$ and with $\head(\tau)$ the \ucqineq $q=\bigvee_{i=1}^m q_i$. 
%
Moreover, we say that $\tau$ is \emph{non-disjunctive} if $m=1$, i.e.\ $\head(\tau)$ is a CQ, and that $\tau$ is \emph{single-head} if it is non-disjunctive and $|\predatoms(\head(\tau))|=1$.
In the rest of the paper, we omit that databases and dependencies are given for a certain signature.

\iflong
\subsubsection{Semantics}
\else
\inlinesubsection{Semantics}
\fi
%
Dependencies and safe \bucqsineq are subclasses of the class of domain-independent relational calculus formulas~\shortcite{AHV95}.
Given a domain-independent sentence $\phi$ of the relational calculus and a set of facts $\D$ over the same signature, we say that $\phi$ 
\emph{evaluates to true (resp., to false) in $\D$} by actually meaning that it evaluates to true (resp., to false) under the interpretation corresponding to the Herbrand model of $\D$.
Moreover, if $\Sigma$ is a set of dependencies, we write $\SD\modelso\phi$ if $\phi$ is satisfied by every model of the FO theory $\constr\cup\D$. As usual in databases, we adopt the Unique Name Assumption, i.e.\ we assume that, in every FO interpretation, different constants are interpreted by different domain elements. Also, we assume that, in every FO interpretation, the predicate $\bot$ has an empty interpretation (i.e.\ $\bot$ always evaluates to false).
Given a database $\D$ and a dependency $\tau$, we say that $\tau$ is \emph{satisfied by $\D$} if it evaluates to true in $\D$.
Given a database $\D$ and a set of dependencies $\constr$, we say that $\D$ is \emph{consistent with $\constr$} if all the dependencies of $\constr$ are satisfied by $\D$, i.e.\ if $\bigwedge_{\tau\in\Sigma}\tau$ evaluates to true in $\D$. From this definition we immediately have that:
\begin{proposition}
\label{pro:consistency-ub}
    Deciding consistency of a database with a set of dependencies is in $\aczero$ \wrt data complexity.
\end{proposition}

Let $\D$ be a set of facts and $q=\exists \vseq{y}\,(\cnj)$ be a \cqineq over the same signature of $\D$.
An \emph{instantiation of $q$ in $\D$} is a substitution $\sigma$ of the variables occurring in $q$ with constants occurring in $\D$ such that 
$(i)$ $\predatoms(\sigma(\cnj))\subseteq\D$;
$(ii)$ no inequality of the form $c\neq c$ (where $c$ is a constant) occurs in $\sigma(\cnj)$. 
Given a \cqineq $q$, an \emph{image of $q$ in $\D$} is any subset
$\{\sigma(\alpha) \mid \alpha\in\predatoms(q)\}$ of $\D$, where $\sigma$ is any instantiation of $q$ in $\D$.
Given a \ucqineq $q$, an \emph{image of $q$ in $\D$} is any image of $q'$ in $\D$, where $q'\in\cq(q)$.

The following property, which is immediate to verify, is implicitly leveraged throughout the work.

\begin{proposition}\label{pro:image-satisfaction}
    Given a database $\D$ and a safe \bucqineq $q$, we have that $q$ evaluates to true in $\D$ iff there exists an image of $q$ in $\D$.
    Moreover, a dependency $\tau$ is satisfied by $\D$ iff, for every instantiation $\sigma$ of $\body(\tau)$ in $\D$, there exists an image of $\head(\sigma(\tau))$ in $\D$.
\end{proposition}

We point out that the images of a \bcqineq with $k$ predicate atoms in a set of facts $\D$ are at most $n^k$ (where $n$ is the cardinality of $\D$) and can be computed in polynomial time \wrt\ the size of the database.



\iflong
\subsection{Subclasses of dependencies}
\else
\inlinesubsection{Subclasses of dependencies}
\fi
%
We say that a dependency of the form (\ref{eqn:dependency}) is \emph{linear} if $|\predatoms(\cnj)|=1$. 
Also, we call \emph{tuple-generating dependency (TGD)} a non-disjunctive dependency having no occurrences of inequality atoms.
Moreover, if $|\cq(q)|=1$ and 
$\vars(q)\subseteq\vars(\cnj)$, then we say that the dependency is \emph{full non-disjunctive}, or simply \emph{full}. If $\head(\tau)=\bot$, then $\tau$ is called \emph{denial}. Note that every denial dependency is full.

Throughout the paper, we assume w.l.o.g.\ that each full dependency $\tau$ is such that $\head(\tau)$ consists of a single predicate atom. Observe indeed that:
\begin{itemize}
    \item it is always possible to equivalently represent a full dependency through a set of full dependencies with a single atom in their head (such a set contains one dependency for each atom in the head of the initial dependency);
    \item each $\tau$ of the form
    $\forall \vseq{x}\, (\varphi(\vseq{x}) \ra x \ne t)$\footnote{We assume w.l.o.g.\ that no inequality of the form $t\neq t$ or $c\neq d$, where $c$ and $d$ are distinct constants, appears in $\tau$.} (with $x \in \vseq{x}$) can be rewritten as
    $\forall \vseq{x}\,( \sigma(\varphi(\vseq{x})) \ra \bot)$, 
    where $\sigma$ is the substitution $\{x \mapsto t\}$.
\end{itemize}

Given a set of dependencies $\constr$, we call \emph{dependency graph} of $\constr$ the directed graph $\G(\constr)$ whose vertices are the dependencies of $\constr$ and such that there is one edge from the vertex $\tau_1$ to $\tau_2$ iff the head of $\tau_1$ contains an atom whose predicate appears in a predicate atom of the body of $\tau_2$. We say that $\constr$ is \emph{acyclic} if there is no cyclic path in $\G(\constr)$.
Given a set of acyclic dependencies $\constr$, we call \emph{topological order of $\constr$} any topological order of $\G(\constr)$, i.e.\ a sequence $\tup{\tau_1,\ldots,\tau_h}$ of the dependencies of $\constr$ such that, if $i\geq j$, then the vertex $\tau_i$ is not reachable from the vertex $\tau_j$ in $\G(\constr)$.


We then recall the classes of guarded and sticky dependencies~\shortcite{CGL12,BLMS11,CGP12}.
A dependency $\tau$ is \emph{guarded} if there exists an atom $\alpha$ of $\body(\tau)$ such that 
$\vars(\alpha) = \vars(\body(\tau))$.
%
Obviously, 
if a set of dependencies is linear, it is also guarded.
As usual, in order to introduce the stickiness property, we make use of the following auxiliary \iflong inductive \fi definition. Given a set of dependencies $\constr$, for each dependency $\tau\in\constr$ we say that a variable $x$ occurring in $\body(\tau)$\footnote{We assume w.l.o.g.\ that distinct dependencies use distinct variables.} is \emph{marked} if:
\begin{itemize}\itemsep0em
    \item[$(i)$] there exists an atom of $\head(\tau)$ in which $x$ does not occur, or
    \item[$(ii)$] there exist a dependency $\tau'$ and a predicate $p$ such that $\head(\tau)$ and $\body(\tau')$ contain, respectively, $p(\vseq{t})$ and $p(\vseq{t'})$ and, for every position $i$ of $\vseq{t}$ in which $x$ occurs, the $i$-th term of $\vseq{t'}$ is a marked variable.
\end{itemize}
Then, we say that $\constr$ is \emph{sticky} if, for every dependency $\tau\in\constr$, every marked variable occurs in $\body(\tau)$ at most once.


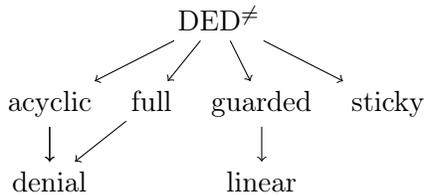
\begin{figure}[t]
\centering
\forestset{ dependencies/.style={for tree={edge={->}}} }
\begin{forest}
dependencies
[\dedineq
    [acyclic,name=A
        [denial,tier=term,name=D]
    ]
    [full,name=F]
    [guarded [linear]]
    [sticky]
]
\draw[->] (A) -> (D);
\draw[->] (F) -> (D);
\end{forest}
\caption{Illustration of hierarchy for sets of dependencies.}
\label{fig:dep-hierarcky}
\end{figure}

We now recall the notion of \emph{chase} for full TGDs.
Formally, given a database $\D$ and a set of full TGDs and denials $\Sigma$, we write $\chase(\D,\constr)$ for indicating the output of the following procedure. First, initialize a set of atoms $\C$ with the content of $\D$. Then, for each dependency $\tau$ in $\constr$, if there exists an instantiation $\subst$ of $\body(\tau)$ in $\C$ such that $\subst(\head(\tau))$ evaluates to false in $\C$, add to $\C$ the (unique and ground) predicate atom of $\subst(\head(\tau))$.
Notably, such a procedure runs in polynomial time w.r.t.\ the size of $\D$.

\medskip

\iflong
\subsection{Complexity classes}

In this article we refer to the following computational complexity classes:
\begin{itemize}
    \item $\aczero$, i.e.\ the class of decision problems solvable by 
    circuit families of constant (i.e.\ independent of the input size) depth, polynomial size, and whose gates have unbounded fanin (each gate can handle and arbitrarily-large number of inputs);
    \item PTIME, i.e.\ the class of decision problems solvable in polynomial time by a deterministic Turing machine;
    \item NP, i.e.\ the class of decision problems solvable in polynomial time by a non-deterministic Turing machine;
    \item coNP, i.e.\ the class of decision problems whose complement is solvable in polynomial time by a non-deterministic Turing machine;
    \item $\pidue$, i.e.\ the class of decision problems whose complement is solvable in polynomial time by a non-deterministic Turing machine augmented by an oracle for some NP-complete problem.
\end{itemize}
These classes are such that $\aczero\subseteq\text{PTIME}\subseteq\text{NP}\cup\text{coNP}\subseteq\pidue$. 
Given a complexity class $C$ and a decision problem $p$, we say that $p$ is \emph{$C$-hard} if there exists a logspace reduction from every other problem in $C$ to $p$. Moreover, if $p$ does also belong to $C$, we say that $p$ is \emph{$C$-complete}.
\else
We assume the reader to be familiar with the basic notions of computational complexity. In this paper we focus on the \emph{data complexity} of the problems studied, i.e.\ the complexity with respect to the size of the database.

\fi
\section{Repairs and decision problems}
\label{sec:repairs}

We are now ready to define the notion of repair, the different entailment semantics and the decision problems that we study.

\begin{definition}[Repairs]\label{def:repair}
    Given a database $\D$ and a set of dependencies $\constr$, a \emph{repair} of $\SD$ is a maximal subset of $\D$ that is consistent with $\constr$.
\end{definition}

Let us call $\reps(\D)$ the set of all the possible repairs of $\SD$.
It is immediate to see that a repair always exists for every $\SD$ (since the empty database satisfies every dependency).
Moreover, we denote as $\intrep(\D)$ the set $\bigcap_{\D'\in \reps(\D)} \D'$.

The first decision problems that we define are related to two distinct notions of entailment of formulas \wrt the repairs of a database.

\begin{definition}[\ars and \irs entailment]\label{def:entailment}
    Given a database $\D$, a set of dependencies $\constr$, and a \bucqineq $q$, we say that:
    \iflong\begin{itemize}\fi
        \iflong\item [$(i)$]\else$(i)$\ \fi
        \emph{$\SD$ entails $q$ under the \ars semantics} (or \emph{$\SD$ \ars-entails $q$} for short, denoted as $\SD\modelsar q$) if $q$ evaluates to true in every $\D'\in \reps(\D)$;
        \iflong\item [$(ii)$]\else$(ii)$\ \fi
        \emph{$\SD$ entails $q$ under the \irs semantics} (or \emph{$\SD$ \irs-entails $q$} for short, denoted as $\SD\modelsiar q$) if $q$ evaluates to true in $\intrep(\D)$.
    \iflong\end{itemize}\fi
\end{definition}

\iflong
The names of these entailment semantics come from the field of Description Logics. In fact, AR and IAR stand for ``ABox repair'' and ``intersection of ABox repairs'', respectively~\cite{LLRRS10}.
As anticipated in the introduction, \irs-entailment is a well-known sound approximation of \ars-entailment. This is illustrated in the subsequent example.
%
\newcommand{\bob}{\mathsf{bob}}
\newcommand{\tom}{\mathsf{tom}}
\newcommand{\ann}{\mathsf{ann}}
\newcommand{\cone}{\mathsf{c}_1}
\newcommand{\ctwo}{\mathsf{c}_2}
\begin{example}
\label{ex:semantics-difference}
    In a university, every lecturer (predicate $L$) of some course which is attended (predicate $A$) by a student is also registered as a teacher (predicate $T$) for that student. Moreover, nobody can be registered as a teacher of herself.
    We also know that Bob is attending the course $\cone$ that is taught by Tom, while Ann is attending the course $\ctwo$ that is taught by herself.
    All this can be modeled through the following database and set of acyclic+full dependencies:
    \[
    \begin{array}{r@{}l@{~~~~}r@{}l}
        \D = \{ & A(\bob,\cone), L(\tom,\cone), T(\tom,\bob), &
        \constr = \{ & \forall l,c,s\, (L(l,c) \land A(s,c) \rightarrow T(l,s)),\\
        & A(\ann,\ctwo), L(\ann,\ctwo), T(\ann,\ann)\}, &
        & \forall t\,(T(t,t) \rightarrow \bot) \} .
    \end{array}
    \]
    The second dependency of $\constr$ is not satisfied by $\D$, and for solving such inconsistency it is necessary to remove the fact $T(\ann,\ann)$. Now, due to the first dependency, we must also remove one between $A(\ann,\ctwo)$ and $L(\ann,\ctwo)$.
    Thus, we have that:
    \[
    \begin{array}{r@{}l}
         \reps(\D)=\{&
        \{A(\bob,\cone), L(\tom,\cone), T(\tom,\bob), A(\ann,\ctwo)\}, \\&
        \{A(\bob,\cone), L(\tom,\cone), T(\tom,\bob), L(\ann,\ctwo)\}
    \}
    \end{array}\] and
    $\intrep(\D)=\{ A(\bob,\cone), L(\tom,\cone), T(\tom,\bob) \}$.
    Let us now consider the query $q=\exists x\,A(x,\ctwo) \lor \exists x\,L(x,\ctwo)$.
    We have that $q$ is entailed by all the repairs, but not by their intersection, i.e.\ it is \ars-entailed but not \irs-entailed by $\SD$.
    \qedexample
\end{example}
\else 
It is easy to verify that, even in the presence of a single dependency with two predicate atoms in the body, the two semantics are different. In particular, the latter is a sound approximation of the first one.
\fi 

In the rest of the paper, we study entailment of safe \bucqsineq under both semantics. We refer to this problem as \emph{instance checking} in the case when the query is a fact.

The following properties, whose proof is immediate, show two important correspondences between the \ars and the \irs semantics of entailment. The second one is a direct consequence of~\shortcite[Corollary~4.3]{CFK12}.

\begin{proposition}
\label{pro:instance-checking-correspondence}
Instance checking under the \ars semantics coincides with instance checking under the \irs semantics.
\iflong
\end{proposition}
\begin{proposition}
In \else Moreover, in \fi
\label{pro:cq-entailment-correspondence}
the case of linear dependencies, entailment under the \ars semantics coincides with entailment under the \irs semantics.
\end{proposition}

We now define another pivotal decision problem studied in the context of CQA.
\iflong\begin{definition}[Repair checking]
\label{def:repair-checking}\fi
Given a set of dependencies $\constr$ and two databases $\D$ and $\D'$, we call \emph{repair checking} the problem of deciding whether a $\D'$ is a repair of $\SD$.
\iflong\end{definition}\fi

\ifrecoverability
Finally, we introduce a further decision problem, called \emph{recoverability}, which is related to the notion of consistency of a database.
\begin{definition}[Recoverability]
\label{def:recoverability}
Given a set of dependencies $\constr$, we say that a subset $\D'$ of $\D$ is \emph{recoverable from $\SD$} if there exists a subset $\D''$ of $\D$ such that $\D'\subseteq\D''$ and $\D''$ is consistent with $\constr$.
\end{definition}
\fi

From now on, in all the decision problems we study, we assume w.l.o.g.\ that all the predicates occurring in $\D$ also occur in $\constr$.
\iflong
\section{Recoverability}
\label{sec:recoverability}
\else
\subsection{Auxiliary definitions and results}
\fi

\ifrecoverability
In this section we establish the upper bounds of the problem of checking recoverability of a database. It is easy to see that, in the general case, the problem belongs to NP.
Anyway, the problem becomes tractable (indeed, it is in $\aczero$) if the acyclic property is accompanied by the full or linear one. In order to prove this result, we provide two first-order rewriting techniques.
We also prove that the problem can be solved in polynomial time if the considered dependencies are linear or full and, more specifically, it belongs to $\aczero$ if both such properties are enjoyed.

\else

In what follows, we present some auxiliary outcomes that are helpful in proving certain complexity results presented further on.
In particular, we introduce the \emph{recoverability} property, that is related to the notion of consistency of a database.
\begin{definition}[Recoverability]
\label{def:recoverability}
    Given a set of dependencies $\constr$, we say that 
    a set of facts $\D'$ is \emph{recoverable from $\SD$} if there exists a subset $\D''$ of $\D$ such that $\D'\subseteq\D''$ and $\D''$ is consistent with $\constr$.
\end{definition}
\fi

Intuitively, a set of facts $\D'$ is recoverable from $\SD$ if it is a subset of $\D$ that can be ``repaired'' adding facts belonging to the original database $\D$. Obviously, if $\D'$ is recoverable from $\SD$ then it is contained in some repair of $\SD$.
For some classes of dependencies, we then study the associated decision problem of checking whether a set of facts $\D'$ is recoverable from $\SD$.

\ifrecoverability
From Proposition~\ref{pro:consistency-ub} and Definition~\ref{def:recoverability}, it immediately follows that:

\begin{lemma}\label{lem:recov-ub-general}
\recovubStatement{arbitrary}{NP}
\end{lemma}
\fi

We now turn our attention to full dependencies.
We recall that the chase of a database w.r.t.\ full dependencies
can be computed in polynomial time. Consequently, the following property holds.

\begin{lemma}\label{lem:recov-ub-full}
\recovubStatement{full}{PTIME}
\end{lemma}
\begin{proof}
    It is straightforward to verify that a database $\D'$ is recoverable from $\SD$ iff $\chase(\D',\Sigma) \subseteq \D$. Then, the thesis follows from the fact that, for full dependencies, $\chase(\D',\Sigma)$ is finite and can be computed in polynomial time.
\end{proof}

Consider now the following crucial property. 

\begin{definition}[CQ-FO-rewritability]
\label{def:cq-fo-rewritable}
    We call a set of dependencies $\Sigma$ \emph{CQ-FO-rewritable} if, 
    for every CQ $q(\vseq{x})$ (where $\vseq{x}$ are the free variables of the query) there exists an FO query $\phi_q(\vseq{x})$ such that, for every database $\D$ and for every tuple of constants $\vseq{t}$ of the same arity as $\vseq{x}$, $\SD\modelso q(\vseq{t})$ iff $\phi_q(\vseq{t})$ evaluates to true in $\D$.
\end{definition}


It is easy to see that CQ-FO-rewritability is enjoyed by any set of full dependencies if such a set is also acyclic, or linear, or sticky, as shown below.

\begin{proposition}
\label{pro:full-fo-rewritable}
    Every set of acyclic+full dependencies, every set of full+linear dependencies and every set of full+sticky dependencies is CQ-FO-rewritable.
\end{proposition}
\begin{proof}
    To prove the proposition, we introduce some auxiliary definitions.
    Given a single-head full dependency $\tau$, we define $\tau_t$ 
    as the TGD obtained from $\tau$ deleting all the inequality atoms in its body.
    Moreover, given a set of full dependencies $\Sigma$, we denote by $\TGD(\Sigma)$ the set of full TGDs $\{\tau_t\mid\tau\in\Sigma\}$.
    
    Now, the key observation is that it is very easy to modify existing rewriting techniques for conjunctive queries over TGDs to deal with single-head full dependencies without inequalities. E.g.\ the algorithm shown in~\shortcite{KonigLMT15} can very easily be adapted to such single-head full dependencies with inequality atoms in rule bodies. Indeed, since the dependencies are full, no variable is deleted by the rewriting steps, which implies that the rewriting of every CQ over $\Sigma$ can be obtained essentially in the same way as in the case of the corresponding $\TGD(\Sigma)$: the additional inequality atoms can just be collected along the rewriting steps, and returned in the final rewritten query (which is thus a UCQ with inequalities rather than a simple UCQ). 
    This immediately implies the following property.
    
    \begin{lemma}
    Let $\Sigma$ be a set of full dependencies. If $\TGD(\Sigma)$ is CQ-FO-rewritable, then $\Sigma$ is CQ-FO-rewritable.
    \end{lemma}
    
    Now, since acyclic TGDs, linear TGDs and sticky TGDs have the CQ-FO-rewritability property (see e.g.~\shortcite{CGL12,CGP12}),
    the above lemma immediately implies that the CQ-FO-rewritability property holds for every set of dependencies that is either acyclic+full, or full+linear, or full+sticky.
\end{proof}

We now prove that, for the above classes of dependencies, checking recoverability is FO rewritable, i.e.\ it is always possible to build an FO formula (depending only on $\constr$) whose evaluation decides if $\D'$ is recoverable from $\SD$.
To this aim, in order to represent the subset $\D'$ of $\D$, we make use of an auxiliary alphabet that contains a predicate symbol $\auxp{p}$ for every $p\in\pred(\constr)$, such that the predicates $\auxp{p}$ and $p$ have the same arity.
Given a formula $\phi$, we denote by $\aux(\phi)$ the formula obtained from $\phi$ replacing every predicate atom $p(\vseq{t})$ with $\auxp{p}(\vseq{t})$, i.e.\ the atom obtained from $p(\vseq{t})$ by replacing its predicate with the corresponding predicate in the auxiliary alphabet.
Similarly, given a set of atoms $\A$, we denote by $\aux(\A)$ the set $\{\aux(\alpha)\mid \alpha\in\A\}$.


Now, let $\Sigma$ be a CQ-FO-rewritable set of dependencies.
We define the following formula:
 \[
 \cqforecovformula(\Sigma)= \bigwedge_{p\in\pred(\constr)}\forall \vseq{x}\, (q_p^\Sigma(\vseq{x})\rightarrow p(\vseq{x}))
 \]
where $\vseq{x}=x_1,\ldots,x_k$ if $k$ is the arity of $p$, and $q_p^\Sigma(\vseq{x})$ is the FO-rewriting of the atomic query $p(\vseq{x})$ (with free variables $\vseq{x}$) with respect to $\Sigma$, expressed over the auxiliary predicate symbols.
Intuitively, by evaluating this sentence over a database $\D\cup\aux(\D')$ one is checking whether $\D$ contains all the facts that are logical consequences of $\Sigma\cup\D'$.

For checking recoverability of a database, we have to slightly modify the above formula as follows.
Given a set of atoms $\A$ over the auxiliary predicates, we define $\cqforecovformula(\A,\Sigma)$ as the formula (with free variables $\vars(\A)$) obtained from $\cqforecovformula(\Sigma)$ by replacing every occurrence of every auxiliary atom $\auxp{p}(\vseq{t})$ with the formula $(\auxp{p}(\vseq{t})\vee \bigvee_{\auxp{p}(\vseq{t'})\in\A} \vseq{t}=\vseq{t'}$), where $\vseq{t}=\vseq{t'}$ is the conjunction of equalities between the terms of $\vseq{t}$ and $\vseq{t'}$.
For example, if $\cqforecovformula(\Sigma)$ is the sentence $\forall x_1,x_2\, (\auxp{s}(x_1,x_2) \vee \auxp{r}(x_2,x_1)\rightarrow r(x_1,x_2))$, then $\cqforecovformula(\{\auxp{r}(y_1,y_2)\},\Sigma)$ is the sentence 
$\forall x_1,x_2\, \big(\auxp{s}(x_1,x_2) \vee (\auxp{r}(x_2,x_1)\vee (y_1=x_2\wedge y_2=x_1))\rightarrow r(x_1,x_2)\big)$.
For catching the intuition behind such a formula, suppose all the atoms of $\A$ are ground. Then, the equalities introduced in $\cqforecovformula(\Sigma)$ are used for simulating the presence of the facts from $\A$ in the database the formula is evaluated over. In this way, evaluating $\cqforecovformula(\A,\Sigma)$ over $\D\cup\aux(D')$ is equivalent to evaluating $\cqforecovformula(\Sigma)$ over $\D\cup\aux(D')\cup\A$.

It is now straightforward to verify that the following property holds.
%
%
\begin{lemma}
\label{lem:chase-wcformula}
    Let $\Sigma$ be a CQ-FO-rewritable set of full dependencies, and let $\D$ and $\D'$ be two databases. Then, $\cqforecovformula(\Sigma)$ evaluates to true in $\D\cup\aux(\D')$ iff $\chase(\D',\Sigma)\subseteq\D$.
    Moreover, for every set of atoms $\A$ and for every instantiation $\sigma$ of $\bigwedge_{\alpha\in\A} \alpha$ in $\D$,
    $\sigma(\cqforecovformula(\A,\Sigma))$ evaluates to true in $\D\cup\aux(\D')$ iff $\chase(\D'\cup\sigma(\A),\Sigma)\subseteq\D$.
\end{lemma}
\iflong
\begin{corollary}
\label{cor:recov-ub-full-fo-rewritable}
    Let $\Sigma$ be a CQ-FO-rewritable set of full dependencies. For every pair of databases $\D,\D'$, deciding whether $\D'$ is recoverable from $\SD$ is in $\aczero$ \wrt data complexity.
\end{corollary}

An example of application of the rewriting function $\cqforecovformula(\cdot)$ is illustrated below.
%
%
\iflong
\begin{example}
\label{ex:recov-acyclic-full}
    %
    Let us consider the database $\D$ and the set of acyclic+full dependencies $\Sigma$ described in Example~\ref{ex:semantics-difference}, for which the sentence $\cqforecovformula(\Sigma)$ is equal to:
    \[
    \begin{array}{l}
        \forall l,c,s\,\big(\auxp{L}(l,c) \wedge \auxp{A}(s,c) \ra T(l,s)\big) \wedge 
        \forall t\,\big(\auxp{T}(t,t) \vee \exists c\,\big(\auxp{L}(t,c) \land \auxp{A}(t,c)\big) \ra \bot\big) .
    \end{array}
    \]
    Now, let us take two subsets of $\D$, namely $\D'=\{A(\bob,\cone), L(\tom,\cone)\}$ and $\D''=\{A(\ann,\ctwo), L(\ann,\ctwo)\}$. One can verify that $\cqforecovformula(\Sigma)$ evaluates to true in $\D\cup\aux(\D')$, but it evaluates to false in $\D\cup\aux(\D'')$. Indeed, we have that $\D'$ is recoverable from $\SD$, while $\D''$ is not.
    \qedexample
\end{example}
\fi


\iflong
For the rest of this section, we focus on linear dependencies.
We first present the algorithm \algcomputerepairlinear (Algorithm~\ref{alg:compute-repair-linear}), which extends the one presented in~\shortcite[Theorem~4.4]{CFK12}. Such a procedure, given a set of linear dependencies $\Sigma$ and a database $\D$, computes the unique repair of $\SD$.

\begin{algorithm}
\caption{\algcomputerepairlinear}
\label{alg:compute-repair-linear}
\begin{algorithmic}[1]
\REQUIRE A database $\D$, a set of dependencies $\constr$;
\ENSURE A database $\D'$;
\STATE $\D' \gets \D$; \\
\REPEAT
    \FORALL {dependencies $\tau\in\Sigma$}
        \FORALL {instantiations $\sigma$ of $\body(\tau)$ in $\D'$ \\
        \quad such that $\sigma(\head(\tau))$ has no image in $\D'$}
            \STATE $\D'\gets\D'\setminus\predatoms  (\sigma(\body(\tau)))$;
        \ENDFOR
    \ENDFOR
\UNTIL {a fixpoint for $\D'$ is reached};
\RETURN $\D'$;
\end{algorithmic}
\end{algorithm}

It is immediate to verify that the algorithm is correct and runs in polynomial time.
It is also evident that a subset $\D'$ of $\D$ is recoverable from $\SD$ iff $\D'$ is a subset of the unique repair of $\SD$.
Consequently, we derive the following upper bound.

\begin{lemma}\label{lem:recov-ub-linear}
\recovubStatement{linear}{PTIME}
\end{lemma}

\fi


\iflong



Finally, we analyze the case of acyclic+linear dependencies, and show that the problem of checking recoverability is in $\aczero$ for such a class of dependencies.
In a similar way as above, we aim to define an FO sentence that decides recoverability of a database.

Given a set of acyclic+linear dependencies $\Sigma$, where $\tup{\tau_1,\ldots,\tau_h}$ is a topological order of $\Sigma$, and an atom $p(\vseq{t})$, we define the formula $\recoval(p(\vseq{t}),\Sigma)$ as follows: 
$(i)$ if $\Sigma=\emptyset$, then $\recoval(p(\vseq{t}),\Sigma)=\true$;
$(ii)$ if $\Sigma\neq\emptyset$, then $\recoval(p(\vseq{t}),\Sigma)$ is the formula:
{\scriptsize
\[
\begin{array}{l}
\displaystyle\hspace{-0.3em}
    \bigwedge_{\footnotesize\begin{array}{c}
        \tau_i\in\Sigma \wedge\\ \predatoms(\body(\tau_i))=\{p(\vseq{t'})\}
    \end{array}}
    \hspace{-0.5em}
    \forall \vseq{x}\,\left( \body(\tau_i) 
    \land \vseq{t}=\vseq{t'} \ra
    \bigg( \bigvee_{\exists \vseq{y}\,\cnj\in\cq(\head(\tau_i))} \exists \vseq{y}\,\Big(\cnj \wedge
    \bigwedge_{\beta\in\predatoms(\cnj)}\recoval(\beta,\{\tau_{i+1},\ldots,\tau_{h}\}) \Big) \bigg)\right)
\end{array}
\]
}
where $\vseq{x}=\vars(\body(\tau_i))\subseteq\vseq{t'}$.
Note that each recursive call removes the first dependency from the topological order of $\Sigma$: if the remaining set is empty, then the function returns true, otherwise we look for dependencies in $\Sigma$ whose body's predicate is $p$. For each such dependency $\tau_i$, we impose the equality between the terms of $p(\vseq{t})$ and $\body(\tau_i)$\footnote{Note again that $\vseq{t}=\vseq{t'}$ could also contain an equality between distinct constants, which would simply make false the evaluation of the premise.} and under such a ``binding'' we check if, for at least one disjunct $q$ in the head, both $q$ and the recursive subformula taking as input $\beta$ and the remaining dependencies are true.

Thus, the following holds:
\begin{lemma}
\label{lem:recoval}
    Let $\Sigma$ be a set of acyclic+linear dependencies, 
    let $\alpha$ be an atom and let $\sigma$ be a substitution of the variables occurring in $\alpha$ with constants. Then, for every database $\D$, 
    $\sigma(\alpha)$ is recoverable from $\SD$ iff $\sigma(\alpha)\in\D$ and $\sigma(\recoval(\alpha,\Sigma))$ evaluates to true in $\D$.
\end{lemma}

Then, we define the formula $\recoval(\Sigma)$ as follows:
\[
\begin{array}{l}
    \displaystyle
    \recoval(\Sigma) =
    \bigwedge_{p\in\pred(\constr)}\forall \vseq{x}\, \big(\auxp{p}(\vseq{x})\rightarrow p(\vseq{x}) \wedge \recoval(p(\vseq{x}),\Sigma)\big)
\end{array}
\]

\begin{proposition}
\label{prop:recov-ub-acyclic-linear}
\recovubStatement{acyclic+linear}{$\aczero$}
\end{proposition}
\begin{proof}
    From Lemma~\ref{lem:recoval} it immediately follows that $\D'$ is recoverable from $\SD$ iff $\recoval(\Sigma)$ evaluates to true in $\D\cup\aux(\D')$. Now, since $\recoval(\Sigma)$ is an FO sentence, the thesis follows.
\end{proof}

%
%
\newcommand{\nickone}{\mathsf{yoda}}
\newcommand{\userone}{\mathsf{u}_1}
\newcommand{\usertwo}{\mathsf{u}_2}
\newcommand{\dateone}{\mathsf{d}_1}
\newcommand{\postone}{\mathsf{p}_1}
\newcommand{\posttwo}{\mathsf{p}_2}
\begin{example}\label{ex:recov-acyclic-linear}
    For modeling a social network we use three predicates, namely $U$ for users (accepting an user ID, a nickname and a registration date), $P$ for posts (accepting a post ID and an author), and $L$ for the ``like'' reaction (accepting a user ID and a post ID).
    Only users can like posts, only posts can be liked, and posts must be authored by users.
    Moreover, we know that the user $\userone$ (with nick $\nickone$ and registration date $\dateone$) likes both posts $\postone$, created by herself, and $\posttwo$, whose author $\usertwo$ is not known to be an user.
    We thus consider the following database and  acyclic+linear set of dependencies:
    \[
    \begin{array}{r@{}l}
         \D = \{& U(\userone,\nickone,\dateone), P(\postone,\userone), P(\posttwo,\usertwo), L(\userone,\postone), L(\userone,\posttwo) \} ,\\
        \constr = \{
            &\forall u,p\, (L(u,p) \ra \exists n,r,a\,(U(u,n,r) \wedge P(p,a))),\\&
            \forall p,a\, (P(p,a) \ra \exists n,r\, U(a,n,r))
        \} .
    \end{array}
    \]
    Then, the sentence $\recoval(\Sigma)$ is equal to:
    \[
    \begin{array}{r@{}l}
        \forall u,p\,(&\auxp{L}(u,p) \ra L(u,p) \land \exists n,r,a\,(U(u,n,r) \land P(p,a) \land\\& \forall\, p',a' (P(p',a')\land p=p' \land a=a' \ra \exists n',r'\,U(a,n',r'))))\,\land \\
        \forall p,a\, (&\auxp{P}(p,a) \ra P(p,a) \land \exists n,r\,U(a,n,r)) .
    \end{array}
    \]
    which can be simplified in:
    \[
    \begin{array}{r@{\ }l}
        \forall u,p & (\auxp{L}(u,p) \ra L(u,p) \land \exists n,r,a\,(U(u,n,r) \land P(p,a) \land \exists n',r'\,U(a,n',r')))\,\land \\
        \forall p,a & (\auxp{P}(p,a) \ra P(p,a) \land \exists n,r\,U(a,n,r)) .
    \end{array}
    \]
    Let us now consider two subsets of $\D$, namely $\D'=\{L(\userone,\postone)\}$ and $\D''=\{L(\userone,\posttwo)\}$. One can verify that the evaluation of $\recoval(\Sigma)$ is true in $\D\cup\aux(\D')$ and false in $\D\cup\aux(\D'')$. Indeed, we have that $\D'$ is recoverable from $\SD$, while $\D''$ is not.
    \qedexample
\end{example}

\fi

 

***OLD STUFF:

QUESTA PARTE SU ACYCLIC È SBAGLIATA, NON È IN PTIME MA È NP-completa:

We now analyze acyclic dependencies, and prove that deciding recoverability in this case can be done in PTIME \wrt data complexity. To this aim, we introduce the notion of TGD expansion sequence.

\begin{definition}[TGD expansion sequence]
\label{def:tes}
Given a database $\D$, a set of dependencies $\constr$, and a subset $\D_0$ of $\D$, a TGD expansion sequence (TES) of $\D_0$ in $\SD$ is a sequence $\seq=\tup{W_1,\ldots,W_p}$ where every $W_i$ (called expansion step) is a quadruple of the form $\tup{B_i,\tau_i,\sigma_i,H_i}$ such that, for each $i\in\{1,\ldots,p\}$:

\begin{enumerate}
    \item $B_i \subseteq \D_{i-1}$, $\tau_i$ is a TGD of $\Sigma$ and $\sigma_i$ is a homomorphism such that $\sigma_i(\body(\tau_i))=B_i$ (i.e.\ $B_i$ is an image of $\body(\tau_i)$);
    
    \item either $(i)$ $H_i\in\D\setminus\D_{i-1}$ and $H_i$ is an image of $\sigma_i(\head(\tau_i))$ or $(ii)$ $i=p$, there exists no image of $\sigma_i(\head(\tau_i))$ in $\D$ and $H_i=\bot$.
    
    \item $\D_i=\D_{i-1}\cup\{H_i\}$; 
\end{enumerate}

Moreover, we denote by $\facts(\seq,\D_0)$ the set $\D_p$.
\end{definition}

Informally, a TGD expansion sequence simulates the iterative insertion into $\D_0$ of all the facts that can be generated with some TGD and that belong to its superset $\D$.

Let $\seq$ be a TES of $\D'$ in $\SD$. We say that $\seq$ is \emph{consistent} if $\bot\not\in\facts(\seq,\D_0)$, \emph{inconsistent} otherwise.
Moreover, we say that $\seq$ is \emph{complete} if either $H_p=\bot$ or, for every $\tau_i\in\Sigma$, if $\tau_i$ is applicable\nb{la nozione di ``applicability'' dovrebbe essere definita\label{note:appl}} to $\facts(\seq,\D_0)$ with $\sigma_i$, then $\sigma_i(\head(\tau_i))$ has an image in $\facts(\seq,\D_0)$.
Informally, $\seq$ is complete if either the fact $\bot$ has been generated or no further TGD expansion step can be applied to the set of facts generated by the expansion.

\begin{lemma}
\label{lem:tgd-closure}
Given a database $\D$, a set of dependencies $\constr$, and a subset $\D'$ of $\D$, we have that $\D'$ is recoverable from $\SD$ iff there exists a complete and consistent TES of $\D'$ in $\SD$.
\end{lemma}

\begin{lemma}
\label{lem:acyclic-generation-bounded}
    Let $\Sigma$ be a set of acyclic TGDs and EGDs, let $\D$ be a database, let $\D'\subseteq\D$, and let $\seq$ be a TES of $\D'$ in $\SD$. For every $\alpha\in\facts(\seq,\D')$ (*** this includes the case when $\alpha=\bot$ ***), there exists a TES $\seq'$ of $\D'$ in $\SD$  such that $\alpha\in\facts(\seq',\D')=\facts(\seq,\D')$ and $\alpha$ is generated in the first $k^h$ expansion steps of $\seq'$ (where $k$ is the maximum length of a TGD or EGD in $\Sigma$ and $h$ is the number of TGDs and EGDs in $\Sigma$).
\end{lemma}
\begin{proof}
    The upper bound follows from the observation that, given the acyclicity condition, for every fact $\beta$, if a subset $\TGDs$ of the TGDs of $\Sigma$ have been used to generate $\beta$, the successors of $\beta$ cannot be generated by the TGDs in $\TGDs$. So, there are at most $h$ levels of predecessors of the fact $\alpha$, and since there are at most $k$ predecessors of every fact, the thesis follows. *** FORSE SI PUO' SPIEGARE CON MAGGIORE RIFERIMENTO ALLE TES (è stata scritta prima di introdurre le TES) ***
\end{proof}


Given a database $\D$, a set of dependencies $\constr$, a set of facts $\D_0$ and two TESs $\seq_1$ and $\seq_2$, we say that $\seq_2$ is \emph{incompatible with $\seq_1$} if there exist an expansion step $\tup{B_i,\tau_i,\sigma_i,H_i}$ in $\seq_1$ and an expansion step $\tup{B_j,\tau_j,\sigma_j,H_j}$ in $\seq_2$ such that $H_i\neq H_j$ and $H_i$ is an image of $\sigma_j(\head(\tau_j))$. 

It is immediate to verify that $\seq_2$ is incompatible with $\seq_1$ iff the concatenation $\seq_1\circ\seq_2$ is \emph{not} a TES of $\D_0$ in $\SD$.

We say that two consistent TESs are \emph{equivalent} if they generate the same set of facts.

\begin{lemma}
\label{lem:contained-tes}
If $\seq$ and $\seq'$ are such that $\facts(\seq',\D')\subseteq\facts(\seq,\D')$, then $\seq'$ is compatible with $\seq$.
\end{lemma}

\begin{algorithm}
\caption{\algrecovacyclic}
\label{alg:recoverability-acyclic}
\begin{algorithmic}[1]
\REQUIRE A set of acyclic TGDs and EGDs $\Sigma$, a database $\D$, a subset $\D'$ of $\D$;
\ENSURE A Boolean value;
\STATE \textbf{let} $k$ be the maximum length of a TGD or EGD in $\Sigma$;
\STATE \textbf{let} $h$ be the number of TGDs in $\Sigma$;
\STATE \textbf{if} there exists a sequence of TESs $\tup{\seq_1,\ldots,\seq_\ell}$
\STATE \quad such that
\STATE \quad $(i)$ each $\seq_i$ has no more than $k^h$ expansion steps
\STATE \quad $(i)$ for every $i\in\{1,\ldots,\ell-1\}$
\STATE \qquad $\seq_{i+1}$ is compatible with $\seq_i$
\STATE \quad $(iii)$ for every inconsistent TES $\seq'$ s.t.\ $|\seq'|\leq k^h$
\STATE \qquad there exists $i\in\{1,\ldots,\ell\}$ 
\STATE \qquad such that $\seq'$ is incompatible with $\seq_i$
\STATE \textbf{then return} true;
\STATE \textbf{else return} false;
\end{algorithmic}
\end{algorithm}

\begin{theorem}
\label{thm:acyclic-recoverability}
    Let $\Sigma$ be a set of acyclic TGDs and EGDs, let $\D$ be a database, and let $\D'\subseteq\D$. $\D'$ is recoverable from $\Sigma$ iff the algorithm \algrecovacyclic$(\Sigma,\D,\D')$\nb{qui c'è uno spazio di troppo, lo lasciamo o lo togliamo dalla macro e lo aggiungiamo a mano per ogni utilizzo?} returns true.
\end{theorem}
\begin{proof}
    If $\D'$ is recoverable from $\Sigma$, then by Lemma~\ref{lem:tgd-closure} there exists a complete and consistent TES $\seq$ of $\D'$ in $\SD$. From Lemma~\ref{lem:acyclic-generation-bounded} it follows that for every TES that contains $\bot$, there exists an equivalent TES that generates $\bot$ within the first $k^h$ expansion steps. So, all the possible generations of $\bot$ are represented by the inconsistent TESs whose length is bounded to $k^h$. On the other hand, the reason why $\bot\not\in\facts(\seq,\D')$ is that none of the above inconsistent TESs of bounded length is compatible with $\seq$, which in turn is due to the presence of a fact $\beta$ in $\facts(\seq,\D')$ for each such inconsistent TES. But now, by Lemma~\ref{lem:acyclic-generation-bounded} it follows that, for each such fact, there exists a TES equivalent to $\seq$ that generates $\beta$ within the first $k^h$ expansion steps. So, the union of the prefixes of length $k^h$ of the above TESs equivalent to $\seq$ is a set that satisfies both condition $(i)$ and condition $(iii)$ of the algorithm \algrecovacyclic, while condition $(ii)$ follows from the fact that the above TESs are prefixes of TESs that are pairwise equivalent, and hence compatible by Lemma~\ref{lem:contained-tes}.
    
    For the other direction, suppose there exists a sequence of TESs $\seq_1,\ldots,\seq_\ell$ such that the three conditions of the algorithm \algrecovacyclic hold, and let $\seq''$ be the concatenation of such sequences, i.e.\ $\seq''=\seq_1\circ\ldots\circ\seq_\ell$. Now, $\seq''$ is a TES of $\D'$ in $\SD$, because of condition $(ii)$ of the algorithm. Moreover, due to condition $(iii)$, $\seq''$ is incompatible with all the possible generations of $\bot$: consequently, there exists a complete and consistent TES $\seq'''$ having $\seq''$ as a prefix, hence by Lemma~\ref{lem:tgd-closure} $\D'$ is recoverable from $\Sigma$.
\end{proof}

We are now ready to establish the complexity of checking recoverability in the case of acyclic dependencies:


\begin{lemma}
\label{lem:recov-ub-acyclic}
    Let $\Sigma$ be a set of acyclic dependencies, let $\D$ be a database, and let $\D'\subseteq\D$. Deciding whether $\D'$ is recoverable from $\Sigma$ is in PTIME \wrt data complexity.
\end{lemma}
\begin{proof}
    The proof follows from Algorithm \algrecovacyclic, Theorem~\ref{thm:acyclic-recoverability} and from the fact that computing all the TESs of size not greater than $k^h$ and checking condition $(ii)$ and condition $(iii)$ over such TESs can be done in time polynomial with respect to the size of $\D$.
\end{proof}


Let $\D$ be a database and $\Sigma$ a set of acyclic \ffk dependencies. Then it is possible to write a Boolean FO query that encodes recoverability, i.e.\ a sentence 
that is true in $(\D,\D')$\nb{spiegare cosa significa true in $(\D,\D')$} iff $\D'$ is recoverable from $\SD$. In this formula, we use the predicates $p$ of the initial signature to evaluate atoms over $\D$ and the auxiliary predicates $\auxp{p}$ to evaluate atoms over $\D'$.

A first version of the formula encoding strong inconsistency (i.e, the complement of checking recoverability) simply checks the OR of the dependency violations, i.e.:
\[
\inconsistentprimed(\Sigma)=\bigvee_{\tau\in\Sigma} (\body'(\tau)\wedge\neg\head(\tau))
\]
where $\body'(\tau)$ is the body of $\tau$ expressed using the primed predicates.

However, this formula does not encode the situations in which $\head(\tau)$ has an image in $\D$ and, if we add such an image to $\D'$, then there is a strong violation of another dependency. To also address these cases, the formula should rewrite every dependency.

In the following, we denote with $\phi$ a sentence (corresponding to the negation of a dependency) of the form:
\begin{equation}
\label{eqn:primed-conjunction}
\exists x\, (\cnj'\wedge\cnj\wedge\neg\exists y\, \beta) 
\end{equation}
where $\cnj'$ is a conjunction of primed predicate atoms, $\cnj$ is a conjunction of non-primed predicate atoms, and $\beta$ is a non-primed predicate atom. \warn{(***DOBBIAMO AGGIUNGERE LE DISUGUAGLIANZE***)}

Notice that every disjunct of $\inconsistentprimed(\Sigma)$ has the above form (\ref{eqn:primed-conjunction}).

Let $\tup{\tau_1,\ldots,\tau_h}$ be a topological order\nb{verrà definito solo nella sezione~\ref{sec:qe-upper-bounds}.4} of the dependencies in $\Sigma$. We define:
\[
\begin{array}{l}
\unrecoverable(\Sigma)=\\
\quad\rewr(\rewr(\ldots(\rewr(\rewr(\inconsistentprimed(\Sigma),\tau_1),\tau_2),\ldots,\tau_{h-1}),\tau_h)
\end{array}
\]
where $\rewr(\bigvee_{i\in\{1,\ldots,m\}}\phi_i,\tau)$ is defined as follows:
\[
\rewr(\bigvee_{i\in\{1,\ldots,m\}}\phi_i,\tau) = \bigvee_{i\in\{1,\ldots,m\}} \rewr(\phi_i,\tau)
\]
and $\rewr(\phi,\tau)$ (where $\phi$ has the above form (\ref{eqn:primed-conjunction})) is defined as follows:
\[
\begin{array}{l}
\displaystyle
\rewr(\phi,\tau) =
\quad \phi \vee \bigvee_{\alpha\in\primed(\phi) \wedge \applicable(\tau,\alpha)}
\rewr(\phi,\tau,\alpha)
\end{array}
\]
where $\primed(\phi)$ is the set of primed atoms occurring in $\phi$, $\applicable(\tau,\alpha)$ is true iff $\head'(\tau)$\nb{$\head'$ andrebbe definito come fatto per $\body'$?} and $\alpha$ unify, and
\[
\begin{array}{l}
\rewr(\phi,\tau,\alpha) =
\sigma(\body'(\tau)\wedge\head(\tau)\wedge\phi^-_\alpha)
\end{array}
\]
where $\sigma$ is an MGU for $\head'(\tau)$ and $\alpha$ and $\phi^-_\alpha$ is the formula obtained from $\phi$ by deleting $\alpha$. 

Informally, $\rewr(\phi,\tau,\alpha)$ is obtained by replacing the primed predicate atom $\alpha$ with the formula $\body'(\tau)\wedge\head(\tau)$ and then applying the unification between $\head'(\tau)$ and $\alpha$.


\begin{theorem}\label{thm:fo-rewr-recov-acyclic-ffk-correct}
Let $\Sigma$ be a set of acyclic \ffk dependencies, and let $\D$ and $\D'$ be two databases such that $\D'\subseteq\D$. The sentence $\unrecoverable(\Sigma)$ is true in $(\D,\D')$ iff $\D'$ is not recoverable from $\SD$.
\end{theorem}

\section{Repair checking}
\label{sec:rc}

In this section we study the data complexity of repair checking in the classes of dependencies previously mentioned. 

\subsubsection*{Lower bounds}

We remark that all the lower bounds shown in this section also hold if the dependencies are restricted to be single-head and without inequalities in the head.

\iflong
First, we focus on the lower bounds of the problem of checking if a given database is a repair. 
We prove the repair checking
\else
First, we prove the
\fi
problem to be PTIME-hard in the case when such dependencies are linear+sticky.

\begin{theorem}
\label{thm:rc-lb-linear-sticky}
    \rclbStatement{linear+sticky}{PTIME}
\end{theorem}
\begin{proof}
    We prove the thesis by showing that the HORN SAT problem can be reduced to repair checking.
    
    Let $\phi$ be a set of $n$ ground Horn rules. We assume w.l.o.g.\ that there is at least one rule without head in $\phi$ (i.e.\ a clause with only negated variables).  Let $\phi'$ be obtained from $\phi$ by slightly modifying rules without heads as follows: every rule without head $r$ is replaced by the rule with the same body as $r$ and with the new variable $u$ in the head; moreover, if $\phi$ contains only one rule without head, then we add the rule $u \rightarrow u$ to $\phi'$ (so $\phi'$ always contains at least two rules with $u$ in the head). It is immediate to verify that $\phi$ is unsatisfiable iff $u$ belongs to all the models (and hence to the minimal model) of $\phi'$.
     
    Now, we associate an id $r_i$ to every Horn rule in $\phi'$, with $r_1$ associated to a rule having $u$ in the head.
    Then, let $\phi'(x)$ be the rules of $\phi'$ having the variable $x$ in their head, let $\pv(\phi')$ be the set of propositional variables occurring in $\phi'$ and, for every rule $r_i$, let $\bv(r_i)$ be the set of variables in the body of $r_i$.
    We define $\D$ as the set of facts:
    \[
    \begin{array}{l}
    \displaystyle
    \bigcup_{x\in\pv(\phi')}\{ H(r_{i_1},x,1,2), H(r_{i_2},x,2,3), \ldots, H(r_{i_h},x,h,1) \mid h=|\phi'(x)| \} \:\cup
    \\\displaystyle
    \hspace{0.83em}\bigcup_{r_i\in\phi'}\{ B(r_{i},x) \mid x \in\bv(r_i) \}
    \end{array}
    \]
    
    Finally, we define the following set of linear+sticky dependencies $\Sigma$:
    \[
    \begin{array}{r@{\ }l}
        \forall x,y,z,w & (H(x,y,z,w)\rightarrow \exists v \, B(x,v)) \\
        \forall x,y,z,w & (H(x,y,z,w)\rightarrow \exists v,t \, H(v,y,w,t)) \\
        \forall x,y & (B(x,y)\rightarrow \exists z,w,v \, H(z,y,w,v)) \\
        \forall x,y,z,w & (H(x,y,z,w)\rightarrow H(r_1,u,1,2)) \\
        \forall x,y & (B(x,y)\rightarrow H(r_1,u,1,2))
    \end{array}
    \]
    Now, the (unique) repair of $\SD$ is the set of facts built starting from $\D$ and deterministically removing some atoms according to the dependencies of $\Sigma$.
    Observe that:
    \begin{itemize}
        \setlength\itemsep{0em}
        \item
        because of the first dependency of $\Sigma$, if for a rule there are no more body atoms, then the head atom of that rule is deleted;
        \item
        because of the second dependency, if a head atom is deleted, then all the head atoms with that variable are deleted;
        \item
        because of the third dependency, if there are no head atoms for a variable, then all the body atoms for that variable are deleted;
        \item
        because of the fourth and fifth dependencies, if the atom $H(r,u,1,2)$ is deleted, then all the atoms are deleted.
    \end{itemize}
    
    Therefore, it can be verified that for every variable $x$ that belongs to the minimal model of $\phi'$, no fact of the form $H(\_,x,\_,\_)$ or $B(x,\_)$ belong to the repair of $\SD$. Consequently, the minimal model of $\phi'$ contains $u$ (i.e.\ $\phi$ is unsatisfiable) iff the fact $H(r_1,u,1,2)$ does not belong to the repair of $\SD$, which implies (because of the fourth and fifth rules of $\Sigma$) that $\phi$ is unsatisfiable iff the empty database is the repair of $\SD$.
\end{proof}

Then, we extend the PTIME lower bound shown in~\shortcite[Theorem~5]{AK09} for full dependencies to the case of full+guarded dependencies.

\begin{theorem}
\label{thm:rc-lb-full-guarded}
    \rclbStatement{full+guarded}{PTIME}
\end{theorem}
\begin{proof}
    We prove the thesis using a modification of the reduction from Horn 3-CNF satisfiability shown in the proof of~\shortcite[Theorem~5]{AK09} for the case of full TGDs.
    
    Let $\phi$ be a Horn 3-CNF formula. We denote by $\clauses(\phi)$ the set of clauses of $\phi$, and by $\pv(\phi)$ the set of propositional variables of $\phi$.
    Let $\D$ and $\D'$ be the following sets of facts:
    \[
    \begin{array}{r@{\ }l}
    \D'=& \{ U(x,y) \mid x \in \clauses(\phi), y \in\pv(\phi) \} \,\cup \\
    & \{ P(x,y,z) \mid x,y\rightarrow z \in \clauses(\phi) \} \,\cup \\
    & \{ N(x,y,z) \mid x,y,z\rightarrow \bot \in \clauses(\phi) \}\\
    \D=& \D'\cup\{ T(x), V(x) \mid x\in\pv(\phi) \}
    \end{array}
    \]
    It is crucial to observe that both $\D$ and $\D'$ contain a fact $U(x,y)$ if the $x$ is a propositional variable belonging to a \emph{unit clause} of $\phi$ and $y$ is \emph{any} variable of $\phi$.
    Then, let $\Sigma$ be a set consisting of the following full+guarded dependencies:
    \[
    \begin{array}{r@{\ }l}
    \forall x,y,z & (N(x,y,z) \land T(x) \land T(y) \land T(z) \rightarrow F) \\
    \forall x,y,z & (P(x,y,z) \land T(x) \land T(y) \rightarrow T(z)) \\
    \forall x,y & (V(y) \land U(x,y) \rightarrow T(x)) \\
    \forall x & (T(x)\rightarrow V(x))
    \end{array}
    \]
    
    We now show that $\phi$ is unsatisfiable iff $\D'$ is a repair of $\SD$.
    
    $(\Rightarrow)$ $\phi$ is unsatisfiable only if $\D$ violates at least one of the first three dependencies. Then, starting from $\D$, we can build a repair of $\SD$ as follows.
    \begin{enumerate}
        \setlength\itemsep{0em}
        \item For each violation of the first or second dependency, we choose to delete one $T$-fact in the image of the body of such dependency which causes the violation. Since this is not sufficient for achieving consistency (otherwise $\phi$ would be satisfiable), we enter a loop, and eventually the third dependency will be violated. Then, go to the next step.
        \item If the third dependency is violated, we choose to delete all the $V$-facts. Then, because of the fourth dependency, we also have to remove all the $T$-facts.
    \end{enumerate}
    The obtained repair coincides with $\D'$.
    
    $(\Leftarrow)$ On the other hand, if $\phi$ is satisfiable, then $\D'$ is not a repair of $\SD$ because the database $\D''=\D'\cup\{ T(x),V(x) \mid x\in\mm(\phi) \}$, where $\mm(\phi)$ is the minimal model of $\phi$, is consistent with $\Sigma$.
\end{proof}

Furthermore, the results provided in~\shortcite{CM05} imply that, in the general case, repair checking is coNP-hard. We establish the same lower bound for the case of guarded+sticky dependencies.

\begin{theorem}
\label{thm:rc-lb-guarded-sticky}
    \rclbStatement{guarded+sticky}{coNP}
\end{theorem}
\begin{proof}
    
    We prove the thesis by showing a reduction from CNF SAT. The reduction is obtained by modifying the reduction from 2-QBF shown in Theorem~4.7 of~\shortcite{CM05}.
    
    Let $\Sigma$ be a set consisting of the following guarded+sticky dependencies:
    \[ 
    \begin{array}{r@{\ }l}
        \forall x,y,z,w &
        (R(x,y,z,w)\rightarrow \exists x',y',z'\, R(x',y',z',z)) \\
        \forall x,y,z,y',z' &
        (R(x,0,y,z)\land R(x,1,y',z')\land\raux(x,y,z,y',z')\rightarrow U(x,y,z,y',z'))
    \end{array}
    \]
    
    Now let $\phi=\psi_1\land\ldots\land\psi_n$ be a CNF formula, where every $\psi_i$ is a clause. Let $\pv(\phi)$ be the set of propositional variables occurring in $\phi$.
    Moreover, let $\D'$ and $\D''$ be the following sets of facts:
    \[
    \begin{array}{r@{\ }l}
    \D' =& \{ R(a,1,i,(i \bmod n)+1) \mid a \textrm{ occurs positively in } \psi_i \} \,\cup \\
        &\{ R(a,0,i,(i \bmod n)+1) \mid a \textrm{ occurs negatively in } \psi_i \}\\
    \D''=& \{ \raux(v,i,j,(i \bmod n)+1,(j \bmod n)+1) \mid v\in\pv(\phi) \text{\ and\ } i,j \in \{1,...,n\}\}
    \end{array}
    \]
    and let $\D = \D' \cup \D''$.
    We prove that $\D''$ is a repair of $\SD$ iff $\phi$ is unsatisfiable.
    
    First, it is immediate to see that $\D''$ is consistent with $\Sigma$.
    Now, due to the first dependency of $\Sigma$, if there exists $\D'''$ such that $\D''\subset\D'''\subseteq\D$ and $\D'''$ is consistent with $\Sigma$, then for every conjunct $\psi_i$ there must exist at least one fact from $R$ of the form $R(\_,\_,i,\_)$ in $\D'''$. But, due to the second dependency of $\Sigma$, this is possible if and only if $\phi$ is satisfiable. Consequently, $\D''$ is a repair of $\SD$ iff $\phi$ is unsatisfiable.
\end{proof}

\subsubsection*{Upper bounds}

As shown in~\shortcite[Proposition~4]{AK09}, in the general case \iflong\ (actually, for \guillemotleft every finite set of first-order constraints\guillemotright),\ \fi repair checking is in coNP.
We show that the problem can be solved in PTIME if the set of dependencies is either linear or full, extending the same complexity result proved in~\shortcite{SC10,CFK12} for less expressive dependencies. 
Moreover, we show that it is in $\aczero$ for acyclic sets of dependencies, as well as for full+linear and full+sticky dependencies.




In order to exploit the previous results for recoverability, we start by noticing the following property (whose proof is straightforward):
\begin{proposition}
    Let $\Sigma$ be a set of dependencies, and let $\D$ and $\D'$ be two databases such that $\D'\subseteq\D$. 
    Then, $\D'$ is a repair of $\SD$ iff $\D'$ is consistent with $\Sigma$ and, for every $\alpha\in\D\setminus\D'$, $\D'\cup\{\alpha\}$ is not recoverable from $\SD$.
\end{proposition}
Consequently,
whenever deciding recoverability is in PTIME \wrt data complexity, also repair checking is in PTIME \wrt data complexity.

Thus, from Lemma~\ref{lem:recov-ub-full} and Lemma~\ref{lem:recov-ub-linear} we can prove the following complexity result.

\begin{theorem}
\label{thm:rc-ub-linear-and-full}
    \rcubStatement{either linear or full}{PTIME}
\end{theorem}


We now focus on the case of acyclic dependencies, for which we prove what follows.

\begin{proposition}\label{pro:acyclic-layers}
    Let $\Sigma$ be a set of acyclic dependencies, and let $\D$ and $\D'$ be two databases such that $\D'\subseteq\D$. 
    Then, $\D'$ is a repair of $\SD$ iff $\D'$ is consistent with $\Sigma$ and, for every $\alpha \in \D\setminus\D'$, $\D'\cup\{\alpha\}$ is inconsistent with $\Sigma$.
\end{proposition}
\iflong 
\begin{proof}
We prove that, if $\D'$ is not a repair of $\SD$ and $\D'\subseteq\D$ and $\D'$ is consistent with $\Sigma$, then there exists $\alpha\in\D\setminus\D'$ such that $\D'\cup\{\alpha\}$ is consistent with $\Sigma$ (the other direction of the proof is straightforward). 
So let us assume that $\D'$ is not a repair of $\SD$ and $\D'\subseteq\D$ and $\D'$ is consistent with $\Sigma$.

First, observe that
, since $\D'$ is not a repair of $\SD$, then
there must exist an $\alpha\in\D\setminus\D'$ such that $\D'\cup\{\alpha\}$ is recoverable from $\SD$.

Then, let us stratify the set of predicates according to the dependencies: layer $1$ contains predicates not occurring in the body of any dependency; layer $i$ contains predicates occurring only in the body of dependencies whose head contains only predicates of layers lower than $i$.
Moreover, we say that a dependency $\tau$ is of layer $i$ if no predicate of layer greater than $i$ occurs in the head of $\tau$.

Now, let $i$ be the minimum layer such that there exists $\alpha\in\D\setminus\D'$ such that $\D'\cup\{\alpha\}$ is recoverable from $\SD$ and the predicate of $\alpha$ belongs to layer $i$.
Let $\alpha$ be such a fact whose predicate is of layer $i$ and suppose that $\D'\cup\{\alpha\}$ is not consistent with $\Sigma$. Then, since $\D'$ is consistent with $\Sigma$, it follows that $\D'\cup\{\alpha\}$ violates a dependency of layer $i-1$.
Now observe that, for every layer $j$ such that $j<i$, and for every $\beta\in\D\setminus\D'$ such that $\D'\cup\{\beta\}$ is not consistent with $\Sigma$ and the predicate of $\beta$ belongs to layer $j$, $\D'\cup\{\beta\}$ is not recoverable from $\SD$. 
This implies that $\D'\cup\{\alpha\}$ is not recoverable from $\SD$, giving rise to a contradiction. Thus, $\D'\cup\{\alpha\}$ is consistent with $\Sigma$.
\end{proof}
\fi 

We now show how to decide repair checking by means of an FO sentence. 
As done in Section~\ref{sec:recoverability}, we use an auxiliary predicate $\auxp{p}$ for every $p\in\pred(\Sigma)$.

First, given any set $\Phi$ of domain-independent FO sentences, we define the formula $\consistentprimed(\Phi)$ as follows:
\[
    \consistentprimed(\Phi) = \bigwedge_{\phi\in\Phi} \aux(\phi)
\]
Clearly, the $\consistentprimed(\Phi)$ evaluates to true in a database $\aux(\D)$ iff each $\phi\in\Phi$ does. In particular, given a set of dependencies $\Sigma$, the evaluation of $\consistentprimed(\Sigma)$ over a database $\aux(\D)$ is true iff the latter is consistent with $\Sigma$.

Then, we define the auxiliary formula $\consistentprimed(\Sigma,\auxp{p}(\vseq{x}))$, where $\auxp{p}(\vseq{x})$ is an atom: it is obtained from $\consistentprimed(\Sigma)$ by replacing every atom $\auxp{p}(\vseq{t})$ 
with the subformula $(\auxp{p}(\vseq{t}) \vee \vseq{t}=\vseq{x})$.


\begin{definition}[Repair checking sentence 1]
\label{def:rc-sentence-acyclic}
    Let $\Sigma$ be a set of acyclic dependencies. We define the formula $\checkrepairacyclic(\Sigma)$ as follows:
    \[ \checkrepairacyclic(\Sigma) =
        \consistentprimed(\Sigma) \land \bigwedge_{p\in\pred(\constr)}\forall \vseq{x}\,\big(p(\vseq{x}) \land \neg \auxp{p}(\vseq{x}) \rightarrow \neg\consistentprimed(\Sigma,\auxp{p}(\vseq{x})\big)
    \]
    where $\vseq{x}=x_1,\ldots,x_m$ if $m$ is the arity of $p$.
\end{definition}

Based on this definition and Proposition~\ref{pro:acyclic-layers}, it is possible to prove the following property.

\begin{theorem}
\label{thm:rc-acyclic-fo-rewriting-correct}
    Let $\Sigma$ be a set of acyclic dependencies, and let $\D$ and $\D'$ be two databases. 
    Then, the sentence $\checkrepairacyclic(\Sigma)$ evaluates to true in $\D\cup\aux(\D')$ iff $\D'$ is a repair of $\SD$.
\end{theorem}
\begin{proof}
    First, it is immediate to verify that, for every instantiation $\sigma$ of $p(\vseq{x})$, the sentence $\consistentprimed(\Sigma,\sigma(\auxp{p}(\vseq{x})))$ evaluates to true in $\aux(\D')$ if and only if $\consistentprimed(\Sigma)$ evaluates to true in $\aux(\D') \cup \{\sigma(\auxp{p}(\vseq{x}))\}$.
    Then, from Proposition~\ref{pro:acyclic-layers} and from the fact that $\consistentprimed(\Sigma)$ evaluates to true in $\aux(\D')$ iff $\D'$ is consistent with $\Sigma$, it follows that $\checkrepairacyclic(\Sigma)$ evaluates to true in $\D\cup\aux(\D')$ iff $\D'$ is a repair of $\SD$.
\end{proof}
\begin{corollary}
\label{cor:rc-ub-acyclic}
    \rcubStatement{acyclic}{$\aczero$}
\end{corollary}

%
        
        
%
\begin{example}
\label{ex:rc-acyclic}
    \newcommand{\mamm}{\mathsf{m}}
    \newcommand{\eleph}{\mathsf{e}}
    \newcommand{\ext}{\mathsf{ex}}
    \newcommand{\iucn}{\mathsf{iucn}}
    IUCN is an organization studying the conservation status of species worldwide. If a taxon is a species (predicate $S$), and it has been described (predicate $D$) by IUCN, then we know its conservation status (predicate $C$). Moreover, our dataset contains the fact that both species \textit{Mammuthus primigenius} and \textit{Elephas maximus} have been described by IUCN, and that the first species is extinct.
    This situation can be modeled through the following database and acyclic set of dependencies:
    \[
    \begin{array}{r@{\ }l}
         \D=\{ & S(\eleph), D(\eleph,\iucn), S(\mamm), D(\mamm,\iucn), C(\mamm,\ext)\ \} ,\\
         \constr=\{ & \forall x\, (S(x) \land D(x,\iucn) \ra \exists y\, C(x,y))\ \} .
    \end{array}
    \]
    %
    Then, the sentence $\checkrepairacyclic(\Sigma)$ is equivalent to:
    \[
    \begin{array}{r@{\,}l}
        \neg\exists x & \big(\auxp{S}(x) \land \auxp{D}(x,\iucn) \land \neg\exists y\,\auxp{C}(x,y)\big) \land\\
        \forall x & \big(S(x) \land \neg \auxp{S}(x) \ra
        \exists x'\, ((\auxp{S}(x') \lor x'=x) \land \auxp{D}(x',\iucn) \land \neg\exists y'\,\auxp{C}(x',y'))\big) \land\\
        \forall x,y & \big(D(x,y) \land \neg \auxp{D}(x,y) \ra
        \exists x'\, (\auxp{S}(x') \land (\auxp{D}(x',y') \lor (x'=x \land \iucn=y)) \land \neg\exists y'\,\auxp{C}(x',y'))\big) \land\\
        \forall x,y & \big(C(x,y) \land \neg \auxp{C}(x,y) \ra
        \exists x'\, (\auxp{S}(x') \land \auxp{D}(x',\iucn) \land \neg\exists y'\,(\auxp{C}(x',y') \lor (x'=x \land y'=y)))\big) .
    \end{array}
    \]
    One can verify that $\reps(\D) = \{ \D\setminus\{S(\eleph)\}, \D\setminus\{D(\eleph,\iucn)\} \}$ and that these are the only two possible subsets of $\D$ such that $\checkrepairacyclic(\constr)$ evaluates to true in $\D\cup\aux(\D')$.
    \qedexample
\end{example}








Finally, we address the cases of full+sticky and full+linear dependencies. To this aim, we exploit the formulas for checking recoverability in case $\Sigma$ is a set of CQ-FO-rewritable dependencies.
Indeed, we define the FO sentence $\repaircheckingformula(\Sigma)$ as follows:

\begin{definition}[Repair checking sentence 2]\label{def:rc-sentence}
Given a CQ-FO-rewritable set of dependencies $\Sigma$, we define $\repaircheckingformula(\Sigma)$ as the FO sentence:
\[
    \repaircheckingformula(\Sigma) =
    \cqforecovformula(\Sigma) \land
    \bigwedge_{p\in\pred(\constr)}
    \neg \exists \vseq{y}\, \big(p(\vseq{y}) \land \neg \auxp{p}(\vseq{y}) \land \cqforecovformula(\{\auxp{p}(\vseq{y})\},\Sigma)\big)
\]
\end{definition}

The intuition behind the formula is straightforward: we check if the database is recoverable (which is a necessary condition for it to be a repair) and no further atom not already present in it can be added while preserving its recoverability.

It is now possible to show that $\repaircheckingformula(\Sigma)$ constitutes a sound and complete FO-rewriting for the repair checking problem, i.e.\ such a formula is evaluated to true over the database $\D\cup\aux(\D')$ iff $\D'$ is a repair of $\SD$.
Therefore, the following property holds.

\begin{theorem}
\label{thm:rc-full-cq-fo-rewritable}
    Let $\Sigma$ be a CQ-FO-rewritable set of full dependencies. For every pair of databases $\D,\D'$, deciding whether $\D'$ is a repair of $\SD$ is in $\aczero$ \wrt data complexity.
\end{theorem}
\begin{proof}
    First, we state the following property, whose proof immediately follows from the definition of chase and the definition of repair.
    \begin{lemma}
    \label{lem:chase-full-linear-rc}
        Let $\Sigma$ be a set of full dependencies, and let $\D$ and $\D'$ be two databases. Then, $\D'$ is a repair of $\SD$ iff $\chase(\D',\Sigma)\subseteq\D$ and, for every fact $\alpha\in\D\setminus\D'$, $\chase(\D'\cup\{\alpha\},\Sigma)\not\subseteq\D$.
    \end{lemma}

    Now, given a CQ-FO-rewritable set of full dependencies $\Sigma$, we can encode the conditions of the above lemma in terms of an FO sentence.
    Indeed, Lemma~\ref{lem:chase-wcformula} and Lemma~\ref{lem:chase-full-linear-rc} imply that, for every set of dependencies $\Sigma$ that is full and CQ-FO-rewritable, $\repaircheckingformula(\Sigma)$ evaluates to true in $\D\cup\aux(\D')$ iff $\D'$ is a repair of $\SD$.
    Thus, repair checking is FO-rewritable for every $\Sigma$ that is full and CQ-FO-rewritable and, consequently, the problem is in $\aczero$.
\end{proof}

As an immediate consequence of Theorem~\ref{thm:rc-full-cq-fo-rewritable} and Proposition~\ref{pro:full-fo-rewritable}, we obtain what follows.

\begin{corollary}
\label{cor:rc-ub-full-linear-and-full-sticky}
    \rcubStatement{full+linear or full+sticky}{$\aczero$}
\end{corollary}

\section{\texorpdfstring{\bucqineq}{BUCQ} entailment under \irs semantics}
\label{sec:intrep}

\newcommand{\allatoms}{\atoms}
\newcommand{\allatomsformula}{\Phi^{ni}}
\newcommand{\iarformula}{\Phi^{\mathit{iar}}}

In this section we focus on the problem of \irs-entailment of safe \bucqsineq.
In the rest of the paper, every time we write \bucqsineq we actually refer to safe \bucqsineq.




\subsubsection*{Lower bounds}

First, we are able to show that instance checking is already $\pidue$-hard for the class of guarded+sticky dependencies.

\begin{theorem}\label{thm:ic-lb-guarded-sticky}
    \iclbStatement{guarded+sticky}{$\pidue$}
\end{theorem}
\begin{proof}
    %
    We prove the thesis by showing a reduction from 2-QBF that is obtained by slightly modifying the reduction from 2-QBF shown in Theorem~4.7 of~\shortcite{CM05} (as well as the reduction shown in the proof of Theorem~\ref{thm:rc-lb-guarded-sticky}). 
    
    Let $\Sigma$ be a set consisting of the following guarded+sticky dependencies:
    \[ 
    \begin{array}{r@{\ }l}
    \forall x,y,z,w & (R(x,y,z,w) \ra \exists x',y',z'\, R(x',y',z',z)) \\
    \forall x,y,z,y',z' & (R(x,0,y,z) \wedge R(x,1,y',z') \wedge \raux(x,y,z,y',z') \ra U(x,y,z,y',z'))
    \end{array}
    \]
    
    Now let $\phi=\forall \vseq{x}\,\exists \vseq{y}\,(\psi_1\wedge\ldots\wedge\psi_n)$ be a 2-QBF, where every $\psi_i$ is a clause over the propositional variables $\vseq{x},\vseq{y}$. Let $\pv(\phi)=\vseq{x}\cup\vseq{y}$.
    Moreover, let $\D_1$, $\D_2$, $\D_3$ and $\D$ be the following sets of facts:
    \[
    \begin{array}{l@{\ }l@{\ }l}
    \D_1&=&\{ R(z,1,i,(i \bmod n)+1) \mid z \textrm{ occurs positively in } \psi_i \} \,\cup \\
       &&\{ R(z,0,i,(i \bmod n)+1) \mid z \textrm{ occurs negatively in } \psi_i \}\\
    \D_2&=&\{ \raux(v,i,j,(i \bmod n)+1,(j \bmod n)+1) \mid v\in\pv(\phi) \text{\ and\ } i,j \in \{1,...,n\}\}\\
    \D_3&=&\{ R(x,1,0,0), R(x,0,0,0) \mid x\in\vseq{x} \}\\
    \D&=&\D_1\cup\D_2\cup\D_3\cup\{R(a,a,1,a)\}
    \end{array}
    \]
    
    We prove that $\SD \modelsar R(a,a,1,a)$ iff $\phi$ is valid.
    
    First, for every interpretation $I_\vseq{x}$ of $\vseq{x}$, let $\D(I_\vseq{x})=\{R(x,1,0,0) \mid I_\vseq{x}(x)=1\} \cup \{R(x,0,0,0) \mid I_\vseq{x}(x)=0\}$ (notice that $\D(I_\vseq{x})\subset\D_3$).
    
    Now let us assume that $\phi$ is not valid, and let $I_\vseq{x}$ be an interpretation of $\vseq{x}$ such that there exists no interpretation $I_\vseq{y}$ of $\vseq{y}$ such that $I_\vseq{x}\cup I_\vseq{y}$ satisfies $\psi_1\wedge\ldots\wedge\psi_n$.
    We prove that $\D_2\cup\D(I_\vseq{x})$ is a repair of $\SD$. Indeed, if a fact from $\D_1\cup\{R(a,a,1,a)\}$ is added to such a set, then consistency would be possible only by adding, for every conjunct $\psi_i$, at least one fact of the form $R(\_,\_,i,\_)$ in $\D_1$. But, due to due to the second dependency of $\Sigma$, this is possible if and only if $\phi$ is valid. Consequently, $\D_2\cup\D(I_\vseq{x})$ is a repair of $\SD$, which implies that $\SD$ does not \ars-entail $R(a,a,1,a)$.
    
    Conversely, assume that $\phi$ is valid. Let $I_\vseq{x}$ be any interpretation of $\vseq{x}$. Then, there exists an interpretation $I_\vseq{y}$ of $\vseq{y}$ such that $I_\vseq{x}\cup I_\vseq{y}$ satisfies $\psi_1\wedge\ldots\wedge\psi_n$. Now let 
    $\D(I_\vseq{y})=\{ R(y,z,w,v) \mid y\in\vseq{y} \wedge z=I_\vseq{y}(y) \wedge R(y,z,w,v)\in\D_1 \}$.
    It is immediate to verify that $\D_2\cup\D(I_\vseq{y})$ is consistent with $\Sigma$.
    Now, suppose $\D'$ is a repair of $\SD$ and suppose $R(a,a,1,a)\not\in\D'$. This implies that all the facts of $\D_1$ of the form $R(\_,\_,n,1)$ do not belong to $\D'$ (because $R(a,a,1,a)$ can only violate the first dependency of $\Sigma$), which in turn implies (again due to the first dependency of $\Sigma$) that all the facts of $\D_1$ of the form $R(\_,\_,n-1,n)$ do not belong to $\D'$, and so on: this iteratively proves that \emph{all} the facts of $\D_1$ do not belong to $\D'$. Therefore, $R(a,a,1,a)\not\in\D'$ implies that $\D'\cap\D_1=\emptyset$, that is, there is no fact of the form $R(y,\_,\_,\_)$ with $y\in\vseq{y}$ in $\D'$. But this immediately implies that $\D'\cup\D(I_\vseq{y})$ is consistent with $\Sigma$, thus contradicting the hypothesis that $\D'$ is a repair of $\SD$. Consequently, $R(a,a,1,a)$ belongs to every repair of $\SD$, and therefore $\SD \modelsar R(a,a,1,a)$.
\end{proof}

We can also prove a coNP lower bound for instance checking in the case of acyclic+guarded+sticky dependencies and in the case of full+guarded dependencies.

\begin{theorem}
\label{thm:ic-lb-acyclic}
    \iclbStatement{acyclic+guarded+sticky}{coNP}
\end{theorem}
\begin{proof}
    The proof is obtained through a reduction from 3-CNF.
    Let $\Sigma$ be a set consisting of the following acyclic+guarded+sticky dependencies:
    \[
    \begin{array}{r@{\ }l}
    \forall x & (V(x,1)\wedge V(x,0)\rightarrow U_1(x)) \\
    \forall x,y_1,y_2,y_3,z_1,z_2,z_3 &
    (C_1(x,y_1,z_1,y_2,z_2,y_3,z_3) \wedge
    V(y_1,z_1)\wedge V(y_2,z_2)\wedge V(y_3,z_3) \rightarrow\\
    & \:\: U_2(y_1,z_1,y_2,z_2,y_3,z_3)) \\
    & \ldots \\
    \forall x,y_1,y_2,y_3,z_1,z_2,z_3 &
    (C_7(x,y_1,z_1,y_2,z_2,y_3,z_3) \wedge
    V(y_1,z_1)\wedge V(y_2,z_2)\wedge V(y_3,z_3) \rightarrow\\
    & \:\: U_2(y_1,z_1,y_2,z_2,y_3,z_3)) \\
    \forall x,y_1,y_2,y_3,y_4,y_5,y_6 &
    (C_1(x,y_1,y_2,y_3,y_4,y_5,y_6) \rightarrow\\
    & \:\: \exists w_1,w_2,w_3,w_4,w_5,w_6\: C_2(x,w_1,w_2,w_3,w_4,w_5,w_6)) \\
    & \ldots \\
    \forall x,y_1,y_2,y_3,y_4,y_5,y_6 &
    (C_6(x,y_1,y_2,y_3,y_4,y_5,y_6) \rightarrow\\
    & \:\: \exists w_1,w_2,w_3,w_4,w_5,w_6\: C_7(x,w_1,w_2,w_3,w_4,w_5,w_6)) \\
    & U \rightarrow \exists x,y_1,y_2,y_3,y_4,y_5,y_6\: C_1(x,y_1,y_2,y_3,y_4,y_5,y_6)
    \end{array}
    \]
    Given a 3-CNF formula $\phi$, we define $\D$ as be the database containing:
    $V(a,0),V(a,1)$ for each variable $a$ occurring in $\phi$,
    seven facts $C_1, \ldots, C_7$ for each clause of $\phi$ (each such fact represents an evaluation of the three variables of the clause that make the clause true),
    and the fact $U$.
    E.g.\ if the $i$-th clause is $\neg a \vee \neg b \vee c$, then:
    \[
    \begin{array}{r@{\ }l}
    \D = \{ & \ldots,
      V(a,0), V(a,1),
      V(b,0), V(b,1),
      V(c,0), V(c,1),\\&
      C_1(i,a,0,b,0,c,0),
      C_2(i,a,0,b,0,c,1),
      C_3(i,a,0,b,1,c,0),
      C_4(i,a,0,b,1,c,1),\\&
      C_5(i,a,1,b,0,c,0),
      C_6(i,a,1,b,0,c,1),
      C_7(i,a,1,b,1,c,1),
      \ldots, U\ \}
    \end{array}
    \]
    
    We prove that the fact $U$ belongs to all the repairs of $\SD$ iff $\phi$ is unsatisfiable.
    
    In fact, if $\phi$ is unsatisfiable, then there is no repair of $\SD$ without any fact of $C_1$, consequently $U$ belongs to all the repairs of $\SD$.
    Conversely, if $\phi$ is satisfiable, then there is a repair (where the extension of $V$ corresponds to the interpretation satisfying $\phi$) that for each clause does not contain at least one of the 7 facts $C_1,\ldots,C_7$ representing the clause. Due to the dependencies between $C_i$ and $C_{i+1}$, this implies that there is a repair that for each clause does not contain any fact $C_1$, and therefore (due to the last dependency) it does not contain the fact $U$.
\end{proof}

\begin{theorem}
\label{thm:ic-lb-ffk}
    \iclbStatement{full+guarded}{coNP}
\end{theorem}
\begin{proof}
    The proof is by reduction from 3-CNF.
    We define the following set $\Sigma$ of full+guarded dependencies:
    
    \[
    \begin{array}{r@{\ }l}
    \forall x_1,x_2,x_3,v_1,v_2,v_3,y,z & (S(z)\wedge
    C(y,z,x_1,v_1,x_2,v_2,x_3,v_3)\wedge V(x_1,v_1)\rightarrow S(y)) \\
    \forall x_1,x_2,x_3,v_1,v_2,v_3,y,z & (S(z)\wedge
     C(y,z,x_1,v_1,x_2,v_2,x_3,v_3)\wedge V(x_2,v_2)\rightarrow S(y)) \\
    \forall x_1,x_2,x_3,v_1,v_2,v_3,y,z & (S(z)\wedge
    C(y,z,x_1,v_1,x_2,v_2,x_3,v_3)\wedge V(x_3,v_3)\rightarrow S(y)) \\
    \forall x & (V(x,0)\wedge V(x,1)\rightarrow U)
    \end{array}
    \]
    
    Given a 3-CNF formula $\phi$ with $m$ clauses, in the database $\D$, we represent every clause of $\phi$ with a fact $C$: e.g.\ if the $i$-th clause of $\phi$ is $a\vee \neg b \vee c$, we add the fact $C(i-1,i,a,1,b,0,c,1)$.
    
    Moreover, the database contains the facts $V(p,0),V(p,1)$ for every propositional variable $p$, and the facts $\{ S(1),S(2),\ldots,S(m) \}$.
    
    We prove that $\phi$ is unsatisfiable iff the fact $S(m)$ belongs to all the repairs of $\SD$.
    
    First, if $\phi$ is satisfiable, then let $P$ be the set of propositional variables occurring in $\phi$, let $I$ be an interpretation (subset of $P$) satisfying $\phi$, and let $\D'$ be the following subset of $\D$:
    \[
    \begin{array}{r@{}l}
    \D' = \D \setminus (
        &\{V(p,0) \mid p\in P\cap I \}\;\cup\\
        &\{ V(p,1) \mid p\in P\setminus I \} \cup \{S(1),\ldots,S(m)\} )
    \end{array}
    \]
    
    It is immediate to verify that $\D'$ is consistent with $\Sigma$ and that $\D'\cup \{S(m)\}$ is not recoverable from $\SD$. This is proved by the fact that the addition of $S(m)$ creates a sequence of instantiations of the bodies of the first three dependencies of $\Sigma$ that requires (to keep the consistency of the set) to add to $\D'\cup\{S(m)\}$ first the fact $S(m-1)$, then $S(m-2)$, and so on until $S(1)$: this in turn would imply to also add $S(0)$, but $S(0)$ does not belong to $\D$, which proves that the set $\D'\cup \{S(m)\}$ is not recoverable from $\SD$.
    Consequently, there exists a repair of $\SD$ that does not contain $S(m)$.
    
    On the other hand, given a guess of the atoms of the $V$ predicate satisfying the fourth dependency and corresponding to an interpretation of the propositional variables that does not satisfy $\phi$, it is immediate to verify that the sequence of instantiations of the bodies of the first three dependencies of $\Sigma$ mentioned above (which has previously lead to the need of adding $S(0)$ to the set) is blocked by the absence of some fact for $V$. More precisely: there exists a positive integer $k\leq m$ such that the atoms $S(k), S(k+1), \ldots, S(m)$ can be added to all the repairs corresponding to such a guess of the $V$ atoms. This implies that, if $\phi$ is unsatisfiable, then $S(m)$ belongs to all the repairs of $\SD$.
\end{proof}

Finally, from Theorem~\ref{thm:rc-lb-linear-sticky} above and Theorem~4.1 of~\shortcite{CM05} (which already holds for inclusion dependencies\footnote{Inclusion dependencies are (equivalent to) single-head linear dependencies without inequalities in which no variable in the body can occur multiple times, thus they are also sticky.}) it follows that:
\begin{corollary}
\label{cor:ic-lb-linear}
    \iclbStatement{linear+sticky}{PTIME}
\end{corollary}

\subsubsection*{Upper bounds}

\ifshort
For proving the upper bounds for \irs-entailment, we first introduce the following auxiliary definition.
Given a set of dependencies $\constr$ and a database $\D$, we say that a subset $\D'$ of $\D$ is \emph{recoverable from $\SD$} if there exists a subset $\D''$ of $\D$ such that $\D'\subseteq\D''$ and $\D''$ is consistent with $\constr$.
\fi

\noindent Most of the algorithms that we present for \irs-entailment rely on checking if a set of facts is recoverable.

For the general case we present the Algorithm~\ref{alg:irs} that, given a database $\D$, a set of arbitrary dependencies $\constr$, and a \bucqineq $q$, checks whether $q$ is \irs-entailed by $\SD$.

\begin{algorithm}
\caption{\algirs}
\label{alg:irs}
\begin{algorithmic}[1]
\REQUIRE A set of dependencies $\Sigma$, a database $\D$, a \bucqineq $q$;
\ENSURE A Boolean value;
\STATE \textbf{let} $M_1,\ldots,M_m$ be the images of $q$ in $\D$;
\IF {there exist $m$ subsets $\D_1,\ldots,\D_m$ of $\D$ such that
\\\quad every $\D_i$ is a repair of $\SD$
\\\quad every $M_i\not\subseteq\D_i$}
\RETURN false;
\ELSE\RETURN true;
\ENDIF
\end{algorithmic}
\end{algorithm}

From the definition of \irs semantics, it is easy to prove the correctness of this algorithm.

\begin{proposition}
\label{pro:algorithm-irs-correct}
    The algorithm \algirs{$(\Sigma,\D,q)$} returns true iff the \bucqineq $q$ is \irs-entailed by $\SD$.
\end{proposition}
\begin{proof}
    The above algorithm checks whether, for every image $M_i$ of $q$ in $\D$, there exists a repair $\D_i$ such that $M_i\not\subseteq\D_i$. If this is true, every such $M_i$ is not contained in $\intrep(\D)$, thus $q$ does not evaluate to true in such an intersection, i.e.\ $\SD\not\modelsiar q$.
    Otherwise, the query $q$ has at least one image in $\intrep(\D)$, and hence $\SD\modelsiar q$.
\end{proof}

The above property implies the following upper bounds.
\begin{theorem}
\label{thm:ir-ub-general-full-and-acyclic}
\label{thm:ir-ub-general-and-full}
\label{thm:ir-ub-acyclic}
    \irs-entailment is:
    \begin{itemize}
        \item[$(i)$] in $\pidue$ \wrt data complexity in the general case of arbitrary dependencies;
        \item[$(ii)$] in coNP \wrt data complexity in the case of \ffk dependencies;
        \item[$(iii)$] in coNP \wrt data complexity in the case of acyclic dependencies.
    \end{itemize}
\end{theorem}
\begin{proof}
    Consider algorithm \algirs, and recall that $m$ is bounded by $n^k$, where $n$ is the size of $\D$ and $k$ is the number of predicate atoms of $q$.
    Then, theses $(i)$, $(ii)$ and $(iii)$ follow from Proposition~\ref{pro:algorithm-irs-correct}, other than, respectively, Proposition~4 of~\shortcite{AK09}, Theorem~\ref{thm:rc-ub-linear-and-full} and Corollary~\ref{cor:rc-ub-acyclic}.
\end{proof}

As for linear dependencies, we show the following property (which, by Proposition~\ref{pro:cq-entailment-correspondence}, also holds for \ars-entailment).

\begin{theorem}
\label{thm:qe-ub-linear}
    \qeubStatement{\irs}{linear}{PTIME}
\end{theorem}
\begin{proof}
    The proof follows immediately from Algorithm~\ref{alg:compute-repair-linear}: once computed (in PTIME) the unique repair $\D'$ of $\SD$, the query is then evaluated over $\D'$ (which can be done in $\aczero$).
\end{proof}

Then, we focus on the subclass of acyclic+linear dependencies, and we first show the following property.

\begin{lemma}
\label{lem:qe-linear}
    Let $\Sigma$ be a set of linear dependencies, let $\D$ be a database, and let $q$ be a \bucqineq. Then, $\SD\modelsiar q$ (or, equivalently, $\SD\modelsar q$) iff there exists an image $M$ of $q$ in $\D$ such that, for each $\alpha\in M$, $\{\alpha\}$ is recoverable from $\SD$.
\end{lemma}
\begin{proof}
    If the dependencies of $\constr$ are linear, then the repair of $\SD$ is unique.
    In this case, we have that $(i)$ such image $M$ must be recoverable from $\SD$ and $(ii)$ a set of facts is recoverable from $\SD$ iff every fact contained in it is recoverable from $\SD$.
\end{proof}

Based on the above property, we can use the formula $\recoval(\alpha,\Sigma)$ introduced in Section~\ref{sec:recoverability} to define an FO sentence for deciding query entailment for acyclic+linear dependencies.


\begin{definition}[\irs-entailment sentence 1]
\label{def:irs-ent-sentence}
    Given a set of acyclic+linear dependencies $\Sigma$ and a \bucqineq $q$, we define the following FO sentence:
    \[
    \queryentails(q,\Sigma)=\bigvee_{q'\in\cq(q)}\exists \vseq{x}\, 
    \Big(\cnj\wedge
    \bigwedge_{\alpha\in\predatoms(\cnj)} \recoval(\alpha,\Sigma)\Big)
    \]
    where every $q'\in\cq(q)$ is of the form $\exists\vseq{x}\,(\gamma)$.
\end{definition}

The following theorem is an immediate consequence of the above definition of $\queryentails(q,\Sigma)$, Lemma~\ref{lem:recoval} and Lemma~\ref{lem:qe-linear}.

\begin{theorem}
\label{thm:qe-acyclic-linear-fo-rewriting-correct}
    Let $\Sigma$ be a set of acyclic+linear dependencies, let $\D$ be a database, and let $q$ be a \bucqineq. Then, $\SD\modelsiar q$ (or, equivalently, $\SD\modelsar q$) iff $\queryentails(q,\Sigma)$ evaluates to true in $\D$.
\end{theorem}
\begin{corollary}
\label{cor:qe-ub-acyclic-linear}
    \irs- and \ars-entailment are in $\aczero$ \wrt data complexity in the case of acyclic+linear dependencies.
\end{corollary}

\begin{example}
\label{ex:qe-acyclic-linear}
    Recall the database $\D$ and the acyclic+linear set $\Sigma$ of dependencies of Example~\ref{ex:recov-acyclic-linear}.
    Note that $\D$ is inconsistent with $\constr$.
    Since we are under the conditions of Proposition~\ref{pro:cq-entailment-correspondence}, there is only one minimal way of solving such an inconsistency. 
    Specifically, we must delete the fact $P(\posttwo,\usertwo)$ and, consequently, $L(\userone,\posttwo)$.
    Thus, we have that
    $\reps(\D) = \{ \intrep(\D) \}$, where $\intrep(\D) = \{ U(\userone,\nickone,\dateone), P(\postone,\userone), L(\userone,\postone) \}$.
    
    Now, let us take queries
    $q_1=\exists p\,P(p,\userone)$ and 
    $q_2=\exists u,p,a\, (L(u,p) \land P(p,a) \land u \ne a)$ and their respective FO rewritings, namely:
    \[
    \begin{array}{r@{\ }l@{}l}
        \queryentails(q_1,\constr)=&\multicolumn{2}{@{}l}{\exists p\, (P(p,\userone) \land \exists n,r\,U(\userone,n,r))}\\
        \queryentails(q_2,\constr)=&\exists u,p\, (&L(u,p) \land P(p,a) \land u \ne a \land
        \exists n,r\,U(a,n,r)\\&&
        \exists n,r,a\,(U(u,n,r) \land P(p,a) \land \exists n',r'\,U(a,n',r'))) .
    \end{array}
    \]
    One can verify that $q_1$ evaluates to true in $\intrep(\D)$ as well as $\queryentails(q_1,\constr)$ does in $\D$, and that $q_2$ evaluates to false in $\intrep(\D)$ as well as $\queryentails(q_2,\constr)$ does in $\D$.
    \qedexample
\end{example}

Finally, for the cases of acyclic+full, full+sticky and full+linear dependencies, we prove a general property that holds for every set of full dependencies $\Sigma$ that enjoys the CQ-FO-rewritability property.



Let $\P=\{p_1,\ldots,p_m\}$ be a set of predicates, and let $k$ be a positive integer. We define the set of atoms:
\[
\atoms(\P,k) =
\{ p(\vseq{x}_i) \mid p\in\P \mbox{ and } i\in\{1,\ldots,k\} \} ,
\]
where each $\vseq{x}_i$ is a sequence of $h$ fresh variables, if $h$ is the arity of $p$.

Then, given a CQ-FO-rewritable set of dependencies $\Sigma$, a set of atoms $\A\subseteq\allatoms(\P,k)$ and a \bcqineq $q$ of the form $\exists \vseq{x}\,(\cnj)$ (we assume w.l.o.g.\ that $\vseq{x} \cap \vars(\A)=\emptyset$), we define the FO formula $\allatomsformula(\A,q,\Sigma)$ as follows:
\[
\allatomsformula(\A,q,\Sigma) =
\exists \vseq{y} \: \Big( \Big(\bigwedge_{\alpha\in\A} \alpha \Big) \wedge \cqforecovformula(\A,\Sigma) \wedge \neg \cqforecovformula(\A\cup\predatoms(q),\Sigma) \Big) ,
\]
where $\vseq{y}$ is a tuple containing all the variables occurring in $\A$. 
Note that the variables of $\vseq{x}$ are free.
In words, by existentially closing the above formula and then evaluating it over $\D$ it is possible to check whether there exists a common instantiation for all the atoms of $\A$ such that, under the same assignment, $(i)$ such set of atoms is recoverable and $(ii)$ it is no longer recoverable when we add the atoms of the query. Intuitively, if this happens for some $\A$, then $\SD \not\modelsiar q$ .

\begin{definition}[\irs-entailment sentence 2]
\label{def:irs-ent-sentence-full}
Let $\Sigma$ be a CQ-FO-rewritable set of dependencies, let $q$ be a \bcqineq of the form $\exists \vseq{x}\,(\cnj)$, and let $k$ be a positive integer. Then $\iarformula(q,\Sigma,k)$ is the sentence:
\[
\iarformula(q,\Sigma,k) =
\exists \vseq{x} \: \Big( \cnj \wedge \neg \Big( \bigwedge_{\A\subseteq\allatoms(\pred(\constr),k)} \allatomsformula(\A,q,\Sigma)\Big) \Big) .
\]
\end{definition}

It is now possible to prove that there exists a $k$ such that, for every database $\D$, $\iarformula(q,\Sigma,k)$ evaluates to true in $\D$ iff $\SD \modelsiar q$.
Moreover, the above sentence and property can be immediately extended to the case of \bucqsineq.

\begin{theorem}
\label{thm:ir-full-cq-fo-rewritable}
    Let $\Sigma$ be a CQ-FO-rewritable set of full dependencies, let $q$ be a \bucqineq, and let $\D$ be a database. Deciding whether $\SD$ \irs-entails $q$ is in $\aczero$ \wrt data complexity.
\end{theorem}
\begin{proof}
    First, it is easy to verify that the following property holds.
    \begin{lemma}
    \label{lem:chase-niformula}
        Let $\Sigma$ be a CQ-FO-rewritable set of full dependencies, let $\A$ be a set of atoms, let $q$ be a \bcqineq, let $\D$ be a database, and let $\sigma$ be an instantiation of $q$ in $\D$.
        Then, $\sigma(\allatomsformula(\A,q,\Sigma))$ evaluates to true in $\D$ iff there exists an instantiation $\sigma'$ of $\bigwedge_{\alpha\in\A}\alpha$ in $\D$ such that $\chase(\sigma'(\A),\Sigma)\subseteq\D$ and $\chase(\sigma'(\A)\cup\sigma(\predatoms(q)),\Sigma)\not\subseteq\D$.
    \end{lemma}
    
    Then, the following property follows from the previous lemma and from the fact that, for every database $\D'$ of size not greater than $k$, there exists a subset $\A$ of $\allatoms(\pred(\constr),k)$ such that $\D'$ is an image of $\A$.
    
    \begin{lemma}
    \label{lem:chase-iarformula}
        Let $\Sigma$ be a CQ-FO-rewritable set of full dependencies, let $q$ be a \bcqineq, let $k$ be a positive integer, and let $\D$ be a database.
        Then, $\iarformula(q,\Sigma,k)$ evaluates to true in $\D$ iff there exists 
        an image $M$ of $q$ in $\D$ such that, 
        for every database $\D'$ such that $\D'\subseteq\D$ and $|\D'|\leq k$, 
        if $\chase(\D',\Sigma)\subseteq\D$ then $\chase(\D'\cup M,\Sigma)\subseteq\D$.
    \end{lemma}
    
    Finally, if $\Sigma$ is CQ-FO-rewritable, then for every database $\D$ and for every \bcqineq $q$, if there exists a subset $\D'$ of $\D$ such that $\chase(\D',\Sigma)\subseteq\D$ and $\chase(\D'\cup M,\Sigma)\not\subseteq\D$, then there exists a subset $\D''$ of $\D$ such that $\chase(\D'',\Sigma)\subseteq\D$ and $\chase(\D''\cup M,\Sigma)\not\subseteq\D$ and $|\D''|\leq k$, where $k$ is an integer that depends on $\Sigma$ and $q$ (so it does not depend on $\D$).
    
    This property and Lemma~\ref{lem:chase-iarformula} imply that \bcqineq entailment under \irs semantics is FO-rewritable, and thus in $\aczero$ \wrt data complexity.
    
    The FO-rewritability property straightforwardly extends to the case when $q$ is a \bucqineq, as $\SD \modelsiar q$ iff there exists a \bcqineq $q'\in \cq(q)$ such that $\SD \modelsiar q'$.
\end{proof}

As a consequence of the above theorem and of Proposition~\ref{pro:full-fo-rewritable},
we get the property below.

\begin{corollary}
\label{cor:ir-full-cq-fo-rewritable}
    \qeubStatement{\irs}{either acyclic+full, or full+linear or full+sticky}{$\aczero$}
\end{corollary}

\section{\texorpdfstring{\bucqineq}{BUCQ} entailment under \ars semantics}
\label{sec:allrep}

We now turn our attention to the problem of \ars-entailment of \bucqsineq.


As for the lower bounds for \ars-entailment, almost all of them follow from the ones already shown for the instance checking problem. Here, we show the following property.

\begin{theorem}
\label{thm:ar-lb-acyclic-full-guarded-sticky}
    \qelbStatement{\ars}{acyclic+full+guarded+\allowbreak sticky}{coNP}
\end{theorem}
\begin{proof}
    The following reduction from 3-CNF SAT uses a set of dependencies $\Sigma$ that is acyclic+full+guarded+sticky:
    \[
    \begin{array}{r@{\ }l}
    \forall x & (V(x,0)\wedge V(x,1) \rightarrow U(x)) \\
    \forall z,x_1,v_1,x_2,v_2,x_3,v_3 &
    (\mathit{NC}(z,x_1,v_1,x_2,v_2,x_3,v_3) \rightarrow V(x_1,v_1)) \\
    \forall z,x_1,v_1,x_2,v_2,x_3,v_3 &
    (\mathit{NC}(z,x_1,v_1,x_2,v_2,x_3,v_3) \rightarrow V(x_2,v_2)) \\
    \forall z,x_1,v_1,x_2,v_2,x_3,v_3 &
    (\mathit{NC}(z,x_1,v_1,x_2,v_2,x_3,v_3) \rightarrow V(x_3,v_3))
    \end{array}
    \]
    
    Then, given a 3-CNF formula $\phi$ with $m$ clauses, in the database $\D$ we represent every clause of $\phi$ with a fact $NC$: e.g.\ if the $i$-th clause of $\phi$ is $\neg a \vee b\vee\neg c$, we add the fact $NC(i,a,1,b,0,c,1)$.
    Moreover, the database contains the facts $V(a,0),V(a,1)$ for every propositional variable $a\in\vars(\phi)$.
    
    Now, given the following BCQ: \[
    q = \exists z,x_1,v_1,x_2,v_2,x_3,v_3 \:\mathit{NC}(z,x_1,v_1,x_2,v_2,x_3,v_3)
    \]
    it is easy to verify that $\phi$ is unsatisfiable iff $\SD \modelsar q$.
\end{proof}


Notice that, as a consequence of the previous theorem, both in the case of acyclic+full dependencies and in the case of full+sticky dependencies, \ars-entailment of \bucqsineq is coNP-hard, and therefore computationally harder than \irs-entailment of \bucqsineq (which is in $\aczero$ in both cases).

Moreover, we establish the following upper bounds for \iflong \ars-entailment in \fi the case of arbitrary, full, and acyclic dependencies.

\begin{theorem}
\label{thm:ar-ub-general-and-full-and-linear}
\ars-entailment is
\iflong\begin{itemize}\fi
    \iflong\item[$(a)$]\else$(a)$ \fi
    in $\pidue$ \wrt data complexity in the general case of arbitrary dependencies;
    \iflong\item[$(b)$]\else$(b)$ \fi
    in coNP \wrt data complexity in the case of either full or acyclic dependencies.
\iflong\end{itemize}\fi
\end{theorem}
\begin{proof}
    Let $q$ be a \bucqineq. From the definition of \ars-entailment, it follows that $\SD\not\modelsar q$ iff there exists a subset $\D'$ of $\D$ such that $(i)$ $\D'$ is a repair of $\SD$ and $(ii)$ $q$ evaluates to false in $\D'$. Therefore, the theses $(a)$ and $(b)$ follow from the fact that repair checking is, respectively, in coNP and in PTIME, and that evaluating a \bucqineq can be done in $\aczero$.
\end{proof}
\section{Conclusions}
\label{sec:conclusions}

This article falls within the scope of research on consistent query answering. We employ a notion of repair based on tuple-deletion, and consistency of databases is checked with respect to disjunctive embedded dependencies with inequalities, a very expressive language for schema constraints. We studied different decision problems related to the notion of database repair under tuple-deletion semantics for the whole class of \dedsineq, for the linear, acyclic, full, sticky and guarded subclasses, and for all the possible combinations thereof. We have shown tractability \wrt data complexity of the examined decision problems for several of such classes of dependencies.

\begin{table}[t]
\centering
\setlength\tabcolsep{3pt} 
\begin{tabular}{|l|c|c|c|c|}
\hline
Existential & \multirow{2}{*}{RC} & \multirow{2}{*}{IC} & \bucqineq & \bucqineq \\
rules class & & & $\modelsiar$ & $\modelsar$ \\
\hline
A+L[+S] & in $\aczero$ & in $\aczero$ &  \multicolumn{2}{c|}{in $\aczero$} \\
\hline
F+L[+S] & in $\aczero$ & in $\aczero$ & \multicolumn{2}{c|}{in $\aczero$} \\
\hline
A+F[+G][+S] & in $\aczero$ & in $\aczero$ & in $\aczero$ & \coNP~\lbcm \\
\hline
F+S[+A][+G] & in $\aczero$ & in $\aczero$ & in $\aczero$ & \coNP~\lbcm\ubfsh \\
\hline
L[+S] & \PTIME~\ubwi & \PTIME & \multicolumn{2}{c|}{\PTIME~\ubwi} \\
\hline
A[+G][+S] & in $\aczero$ & \coNP & \coNP & \coNP~\lbcm \\
\hline
F[+G] & \PTIME~\lbak\ubsc & \coNP & \coNP & \coNP~\lbcm\ubfsh \\
\hline
All[+G][+S] & \coNP~\lbcm\ubak & \piduecomp~\lbcm & \piduecomp~\lbcm & \piduecomp~\lbcm\ubwi \\
\hline
\multicolumn{5}{c}{\small\makecell[l]{ \\[-1mm]
    A = Acyclic, F = Full, G = Guarded, L = Linear, S = Sticky 
    %
    \smallskip
}}
\end{tabular}
\caption{Data complexity results for all the classes of existential rules considered in this work.
Joined cells indicate that the two entailment problems coincide (see Proposition~\ref{pro:cq-entailment-correspondence}).
All entries are completeness results, except for the $\aczero$ ones. And the {[+X]} notation indicates that the same complexity holds if the considered class \iflong of dependencies \fi is also a subclass of X.
We recall that L implies G (i.e.\ linear dependencies are a subclass of guarded dependencies).
}
\label{tab:results-all}
\end{table}


Table~\ref{tab:results-all} summarizes all the complexity results established in this paper.
Such findings allow us to draw a complete picture of the data complexity of repair checking and both \ars-entailment and \irs-entailment of BUCQs, for all the possible combinations of the five classes of existential rules considered in this paper.
%
We remark that, with the exception of the cases in $\aczero$, every outcome is a completeness result for the complexity class reported.
Every such result is actually new.
However, some of them (the ones marked with numbers) extend previously known lower and/or upper bounds. More precisely:
\begin{itemize}\itemsep0em
    \item[\lbcmTarget] Extends the lower bound proved in~\shortcite{CM05} for denials;
    \item[\ubwiTarget] Extends the upper bound proved in~\shortcite{CFK12} for non-disjunctive LAV TGDs without inequalities;
    \item[\lbakTarget] Extends the lower bound proved in~\shortcite{AK09} for full TGDs and EGDs;
    \item[\ubscTarget] Extends the upper bound proved in~\shortcite{SC10} for denials;
    \item[\ubfshTarget] Extends the upper bound proved in~\shortcite{CFK12} for GAV TGDs;
    \item[\ubakTarget]  The upper bound has been established in~\shortcite{AK09}.
\end{itemize}



In our opinion, the most important part of our findings is the identification of many classes of existential rules in which the problems studied are in $\aczero$ in data complexity. On the one hand, this is a bit surprising, since consistent query answering is generally known to be a computationally hard task.
On the other hand, for the problems in $\aczero$ we have actually shown that such problems can be solved by the evaluation of an FO sentence over the database: therefore, these results (and in particular Definition~\ref{def:rc-sentence-acyclic}, Definition~\ref{def:rc-sentence}, Definition~\ref{def:irs-ent-sentence} and Definition~\ref{def:irs-ent-sentence-full}) can be a starting point towards the development of practical algorithms for consistent query answering based on FO rewriting techniques.




Another very interesting future research direction of this work is extending our framework towards the combination of open-world assumption (OWA) and closed-world assumption (CWA). Many studies have proposed forms of combination of OWA and CWA in knowledge bases (see e.g. \cite{LSW13,LSW19,AOS20,ALMV18,BB19}). Specifically, for consistent query answering in existential rules the present paper has studied a ``purely closed'' approach, in which all predicates have a closed interpretation, while previous work has considered ``open'' approaches, as explained in the introduction. It would be very interesting to study a hybrid approach, in which the predicates can be partitioned into an ``open'' and a ``closed'' class.
This kind of combination of OWA and CWA is usually very challenging from the computational viewpoint, however both our results and the results for consistent query answering in existential rules under OWA presented in \cite{LMMMPS22} (although for a more restricted language for existential rules than ours) are quite promising.

Finally, consistent query answering has important connections with the problem of \emph{controlled query evaluation}~\shortcite{LRS19}, i.e.\ the problem of evaluating queries on a database (or knowledge base) in the presence of a logical specification of a privacy policy that should not be violated by the query answers. We are very interested in investigating the consequences of our results for such a problem.
In this context, handling partially-complete knowledge is of pivotal importance. As well-explained in~\shortcite{BoSa13,Bona22}, there are several situations where, if an attacker is aware that the system has a complete knowledge over part of the domain, she can exploit this form of \emph{meta-knowledge} to employ non-standard forms of reasoning for inferring sensitive information.

\vskip 0.2in
\bibliographystyle{theapa}
\bibliography{bibliography/strings,bibliography/bibliography}

\end{document}